\pdfoutput=1 
\documentclass[sigconf, 
pbalance,
nonacm,
]{acmart}

\usepackage{array}
\usepackage[ruled]{algorithm2e}
\usepackage{enumitem}
\usepackage{nicefrac}       
\usepackage{pgfplots}
\usepackage{subfig}
\usepackage{tabu}
\usepackage{multirow}
\usepackage{etoolbox}

\usetikzlibrary{patterns, calc, arrows.meta, positioning, shapes, fit, decorations, shapes.multipart, matrix}
\usepgfplotslibrary{fillbetween, colorbrewer, groupplots}
\pgfplotsset{
	compat=newest,
	grid style={line width=.1pt, draw=gray!30},
	cycle list/Dark2,
	cycle multiindex* list={
		mark list*\nextlist
		Dark2\nextlist
	},
	legend image code/.code={
		\draw[mark repeat=2,mark phase=2]
		plot coordinates {
			(0cm,0cm)
			(0.15cm,0cm)        
			(0.3cm,0cm)         
		};%
	}
}

\newcommand{\figwidth}{0.25\textwidth}
\newcommand{\boxsize}{0.15\textwidth}
\newcommand{\ldp}{$\epsilon\text{-LDP }$}
\newcommand{\N}{\mathcal{N}}
\newcommand{\ek}{{\epsilon/m}}
\newcommand{\agg}{\textsc{Aggregate}}
\newcommand{\upd}{\textsc{Update}}
\newcommand{\etal}{\textit{et al}.}
\newcommand\numberthis{\addtocounter{equation}{1}\tag{\theequation}}
\newcommand{\ifequals}[3]{\ifthenelse{\equal{#1}{#2}}{#3}{}}
\newcommand{\case}[2]{#1 #2} 
\newenvironment{switch}[1]{\renewcommand{\case}{\ifequals{#1}}}{}

\newcommand{\map}[1]{
	\begin{switch}{#1}
		\case{cora}{Cora}
		\case{pubmed}{pubmed}
		\case{facebook}{Facebook}
		\case{lastfm}{LastFM}
		\case{gcn}{GCN}
		\case{gat}{GAT}
		\case{sage}{GraphSAGE}
	\end{switch}
}

\DeclareMathOperator{\E}{\mathbb{E}}
\DeclareMathOperator{\csch}{csch}
\DeclareMathSizes{10}{8}{7}{5}

\definecolor{SeaGreen}{rgb}{0.35, 0.63, 0.53}

\newif\ifanonymous
\makeatletter
\if@ACM@anonymous
\anonymoustrue
\else
\anonymousfalse
\fi
\makeatother

\setcopyright{acmcopyright}
\copyrightyear{2021}
\acmYear{2021}
\acmDOI{doi}

\acmConference[CCS'21]{2021 ACM SIGSAC Conference on Computer and Communications Security}{November 14--19, 2021}{Seoul, South Korea}
\acmPrice{15.00}
\acmISBN{978-1-4503-XXXX-X/18/06}

\begin{document}

\title{Locally Private Graph Neural Networks}

\author{Sina Sajadmanesh}
\email{sajadmanesh@idiap.ch}
\orcid{0002-8834-0338}
\affiliation{\institution{Idiap Research Institute}\city{}\country{}}
\affiliation{\institution{EPFL}\city{}\country{}}

\author{Daniel Gatica-Perez}
\orcid{0001-5488-2182}
\email{gatica@idiap.ch}
\affiliation{\institution{Idiap Research Institute}\city{}\country{}}
\affiliation{\institution{EPFL}\city{}\country{}}

\renewcommand{\shortauthors}{Sajadmanesh and Gatica-Perez}

\begin{abstract}
	Graph Neural Networks (GNNs) have demonstrated superior performance in learning node representations for various graph inference tasks. However, learning over graph data can raise privacy concerns when nodes represent people or human-related variables that involve sensitive or personal information. While numerous techniques have been proposed for privacy-preserving deep learning over non-relational data, there is less work addressing the privacy issues pertained to applying deep learning algorithms on graphs. In this paper, we study the problem of node data privacy, where graph nodes have potentially sensitive data that is kept private, but they could be beneficial for a central server for training a GNN over the graph. To address this problem, we develop a privacy-preserving, architecture-agnostic GNN learning algorithm with formal privacy guarantees based on Local Differential Privacy (LDP). Specifically, we propose an LDP encoder and an unbiased rectifier, by which the server can communicate with the graph nodes to privately collect their data and approximate the GNN's first layer. To further reduce the effect of the injected noise, we propose to prepend a simple graph convolution layer, called KProp, which is based on the multi-hop aggregation of the nodes' features acting as a denoising mechanism. Finally, we propose a robust training framework, in which we benefit from KProp's denoising capability to increase the accuracy of inference in the presence of noisy labels. Extensive experiments conducted over real-world datasets demonstrate that our method can maintain a satisfying level of accuracy with low privacy loss.
\end{abstract}

%


\maketitle

\section{Introduction}\label{sec:intro}

In the past few years, extending deep learning models for graph-structured data has attracted growing interest, popularizing the concept of {Graph Neural Networks (GNNs)}~\cite{scarselli2008graph}. GNNs have
shown superior performance
in a wide range of applications in social sciences~\cite{hamilton2017inductive}, biology~\cite{rhee2017hybrid}, molecular chemistry~\cite{duvenaud2015convolutional}, and so on, achieving state-of-the-art results
in various graph-based learning tasks, such as node classification~\cite{kipf2017semi}, link prediction~\cite{zhang2018link}, and community detection~\cite{chen2017supervised}.
However, most real-world graphs associated with people or human-related activities, such as social and economic networks, are often sensitive and might contain personal information. For example in a social network, a user's friend list, profile information, likes and comments, etc., could potentially be private to the user.
To satisfy users' privacy expectations in accordance with recent legal data protection policies,
it is of great importance to develop privacy-preserving GNN models for applications that rely on graphs accessing users' personal data.

\paragraph{\textbf{Problem and motivation.}} In light of these privacy constraints, we define the problem of node data privacy. As illustrated in Figure~\ref{fig:problem}, in this setting, graph nodes, which may represent human users, have potentially sensitive data in the form of feature vectors and possibly labels that are kept private, but the topology of the graph is observable from the viewpoint of a central server, whose goal is to benefit from private node data to learn a GNN over the graph. This problem has many applications in social network analysis and mobile computing. For example, consider a social smartphone application server, e.g., a social network, messaging platform, or a dating app. As this server already has the data about social interactions between its users, the graph topology is not private to the server. However, the server could potentially benefit from users' personal information, such as their phone's sensor data, list of installed apps, or application usage logs, by training a GNN using these private features to learn better user representations for improving its services (e.g., the recommendation system). Without any means of data protection, however, this implies that the server should collect users' personal data directly, which can raise privacy concerns.

\begin{figure}
	\centering
	\includegraphics[width=0.6\columnwidth]{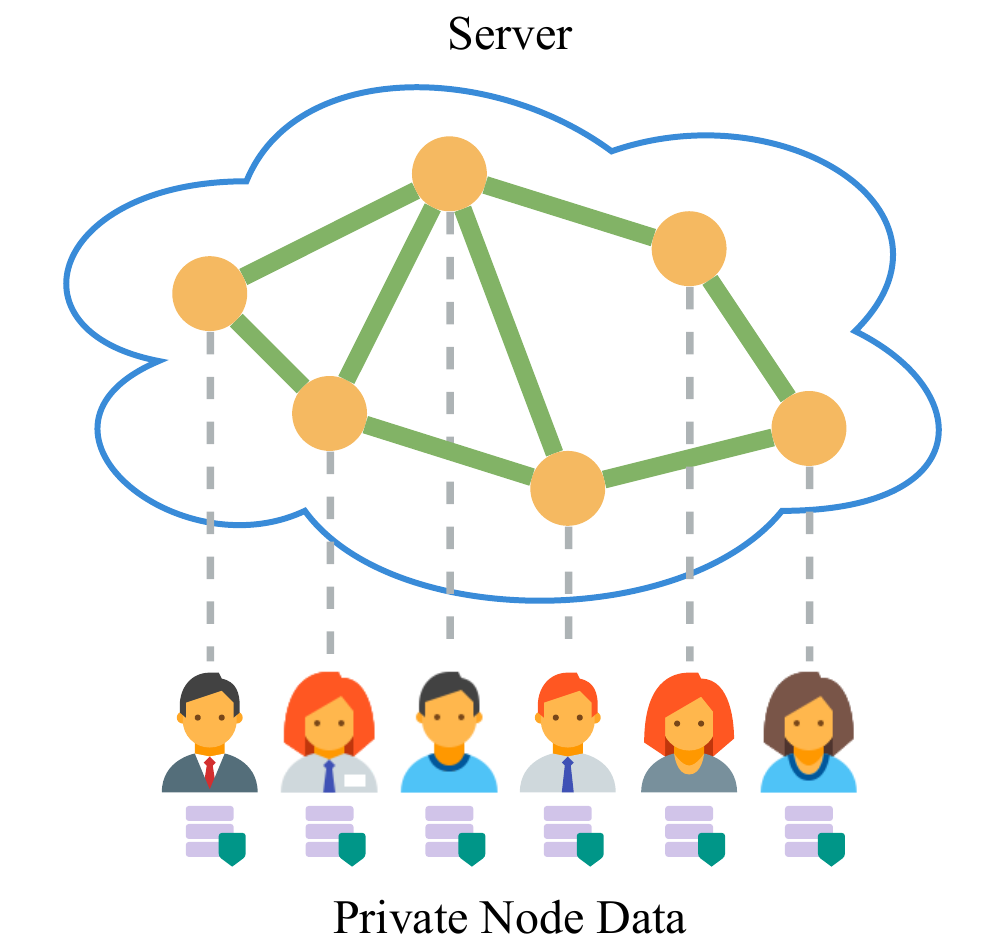}
	\caption{The node data privacy problem. A cloud server (e.g., a social network server) has a graph (e.g., the social graph), whose nodes, which may correspond to real users, have some private data that the server wishes to utilize for training a GNN on the graph, but cannot simply collect them due to privacy constraints.}
	\label{fig:problem}
\end{figure}

\paragraph{\textbf{Challenges.}} Training a GNN from private node data is a challenging task, mainly due to the relational nature of graphs. Unlike other deep learning models wherein the training data points are independent, in GNNs, the samples -- nodes of the graph -- are connected via links and exchange information through the message-passing framework during training \cite{hamilton2017representation}.
This fact renders common collaborative learning paradigms, such as federated learning \cite{kairouz2019advances}, infeasible due to their excessive communication overhead. The main reason is that in the absence of a trusted server, which is the primary assumption of our paper, every adjacent pair of nodes has to exchange their vector representations with each other multiple times during a single training epoch of a GNN, which requires significantly more communication compared to conventional deep neural networks, where the nodes only communicate with the server, independently.


\paragraph{\textbf{Contributions.}} In this paper, we propose the Locally Private Graph Neural Network (LPGNN), a novel privacy-preserving GNN learning framework for training GNN models using private node data. Our method has provable privacy guarantees based on {Local Differential Privacy (LDP)}~\cite{kasiviswanathan2011can}, can be used when either or both node features and labels are private, and can be combined with any GNN architecture independently.

To protect the privacy of node features, we propose an LDP mechanism, called the \emph{multi-bit mechanism}, through which the graph nodes can perturb their features that are then collected by the server with minimum communication overhead. These noisy features are then used to estimate the first graph convolution layer of the GNN. Given that graph convolution layers initially aggregate node features before passing them through non-linear activation functions, we benefit from this aggregation step as a denoising mechanism to average out the differentially private noise we have injected into the node features. To further improve the effectiveness of this denoising process and increase the estimation accuracy of the graph convolution, we propose to prepend a simple yet effective graph convolution layer based on the multi-hop aggregation of node features, called KProp, to the backbone GNN.

To preserve the privacy of node labels, we perturb them using the generalized randomized response mechanism \cite{kairouz2016discrete}. However, learning with perturbed labels introduces extra challenges, as the label noise could significantly degrade the generalization performance of the GNN. To this end, we propose a robust training framework, called Drop (label \textbf{d}enoising with p\textbf{rop}agation), in which we again benefit from KProp's denoising capability to increase the accuracy of noisy labels. Drop can be seamlessly combined with any GNN, is very easy to train, and does not rely on any clean (raw) data in any form, being features or labels, neither for training nor validation and hyper-parameter optimization.

Finally, we derive the theoretical properties of the proposed algorithms, including the formal privacy guarantees and error bound. We conduct extensive experiments over several real-world datasets, which demonstrate that our proposed LPGNN is robust against injected LDP noise, achieving a decent accuracy-privacy trade-off in the presence of noisy features and labels.


\paragraph{\textbf{Paper organization.}} The rest of this paper is organized as follows. In Section~\ref{sec:preliminaries}, we formally define the problem and provide the necessary backgrounds. Then, in Section~\ref{sec:method}, we explain our locally private GNN training algorithm. Details of experiments and their results are explained in Section~\ref{sec:xp}.
We review related work in Section~\ref{sec:review} and finally in Section~\ref{sec:conclusion}, we conclude the paper.
The proofs of all the theoretical findings are also presented in Appendix~\ref{sec:proof}.

\section{Preliminaries}\label{sec:preliminaries}

\paragraph{\textbf{Problem definition.}} We formally define the problem of learning a GNN with node data privacy. Consider a graph $\mathcal{G}=(\mathcal{V}, \mathcal{E}, \mathbf{X}, \mathbf{Y})$, where $\mathcal{E}$ is the link set and $\mathcal{V} = \mathcal{V_L} \cup \mathcal{V_U}$ is the union of the set of labeled nodes $\mathcal{V_L}$ and unlabeled ones $\mathcal{V_U}$. The feature matrix $\mathbf{X}\in\mathbb{R}^{|\mathcal{V}|\times d}$ comprises $d\text{-dimensional}$ feature vectors $\mathbf{x}_v$ for each node $v\in\mathcal{V}$, and $\mathbf{Y}\in\{0,1\}^{|\mathcal{V}|\times c}$ is the label matrix, where $c$ is the number of classes. For each node $v\in\mathcal{V_L}$, $\mathbf{y}_v$ is a one-hot vector, i.e., $\mathbf{y}_v\cdot\vec{\mathbf{1}} = 1$, where $\vec{\mathbf{1}}$ is the all-one vector, and for each node $v\in\mathcal{V_U}$, $\mathbf{y}_v$ is the all-zero vector $\vec{\mathbf{0}}$.
Now assume that a server has access to $\mathcal{V}$ and $\mathcal{E}$, but the feature matrix $\mathbf{X}$ and labels $\mathbf{Y}$ are private to the nodes and thus not observable by the server. The problem is: how can the server collaborate with the nodes to train a GNN over $\mathcal{G}$ without letting private data leave the nodes? To answer this question, we first present the required background about graph neural networks and local differential privacy in the following, and then in the next section, we describe our proposed method in detail. Note that since in our problem setting, nodes of the graph usually correspond to human users, we often use the terms ``node'' and ``user'' interchangably throughout the rest of the paper.

%

\paragraph{\textbf{Graph Neural Networks.}}
A GNN learns a representation for every node in the graph using a set of stacked graph convolution layers. Each layer gets an initial vector for each node and outputs a new embedding vector by aggregating the vectors of the adjacent neighbors followed by a non-linear transformation. More formally,
given a graph $\mathcal{G}=(\mathcal{V}, \mathcal{E}, \mathbf{X})$, an $L\text{-layer}$ GNN consists of $L$ graph convolution layers, where the embedding $\mathbf{h}^{l}_{v}$ of any node $v\in\mathcal{V}$ at layer $l$ is generated by aggregating the previous layer's embeddings of its neighbors, called the neighborhood aggregation step, as:
\begin{align}
	\mathbf{h}^{l}_{\N(v)} & = \agg{}_l\left(\{\mathbf{h}_u^{l-1}, \forall u \in \mathcal{N}(v)\}\right) \label{eq:agg} \\
	\mathbf{h}^{l}_{v}     & = \upd{}_l\left(\mathbf{h}^{l}_{\N(v)}\right) \label{eq:na}
\end{align}
where $\N(v)$ is the set of neighbors of $v$ (which could include $v$ itself) and $\mathbf{h}_{u}^{l-1}$ is the embedding of node $u$ at layer $l-1$. $\agg{}_l(.)$ and $\mathbf{h}^{l}_{\N(v)}$ are respectively the $l\text{-th}$ layer differentiable, permutation invariant aggregator function (such as mean, sum, or max) and its output on $\N(v)$. Finally, $\upd{}_l(.)$ is a trainable non-linear function, e.g., a neural network, for layer $l$. At the very first, we have $\mathbf{h}^0_v = \mathbf{x_v}$, i.e., the initial embedding of $v$ is its feature vector $\mathbf{x}_v$, and the last layer generates a $c\text{-dimensional}$ output followed by a softmax layer to predict node labels in a $c$-class node classification task. 

\paragraph{\textbf{Local Differential Privacy.}}
Local differential privacy (LDP) is an increasingly used approach for collecting private data and computing statistical queries, such as mean, count, and histogram. It has been already deployed by major technology companies, including Google \cite{erlingsson2014rappor}, Apple \cite{thakurta2017learning}, and Microsoft \cite{ding2017collecting}. The key idea behind LDP is that data holders do not need to share their private data with an untrusted data aggregator, but instead send a perturbed version of their data, which is not meaningful individually but can approximate the target query when aggregated. It includes two steps: (i) data holders perturb their data using a special randomized mechanism $\mathcal{M}$ and send the output to the aggregator; and (ii) the aggregator combines all the received perturbed values and estimates the target query. To prevent the aggregator from inferring the original private value from the perturbed one, the mechanism $\mathcal{M}$ must satisfy the following definition~\cite{kasiviswanathan2011can}:
\begin{definition}
	Given $\epsilon > 0$, a randomized mechanism $\mathcal{M}$ satisfies $\epsilon\text{-local differential}$ privacy, if for all possible pairs of user's private data $x$ and $x^\prime$, and for all possible outputs $y\in Range(\mathcal{M})$, we have:
	\begin{equation}
		\Pr[\mathcal{M}(x) = y] \le e^\epsilon\Pr[\mathcal{M}(x^\prime) = y]
	\end{equation}
\end{definition}
The parameter $\epsilon$ in the above definition is called the \emph{``privacy budget''} and is used to tune utility versus privacy: a smaller (resp. larger) $\epsilon$ leads to stronger (resp. weaker) privacy guarantees, but lower (resp. higher) utility. The above definition implies that the mechanism $\mathcal{M}$ should assign similar probabilities (controlled by $\epsilon$) to the outputs of different input values $x$ and $x^\prime$, so that by looking at the outputs, an adversary could not infer the input value with high probability, regardless of any side knowledge they might have. LDP is achieved for a deterministic function usually by adding a special random noise to its output that cancels out when calculating the target aggregation given a sufficiently large number of noisy samples.

%

\section{Proposed Method}\label{sec:method}

In this section, we describe our proposed framework for learning a GNN using private node data.
As described in the previous section, in the forward propagation of a GNN, the node features are only used as the input to the first layer's \agg{} function. 
This aggregation step is amenable to privacy, as it allows us to perturb node features using an LDP mechanism (e.g., by injecting random noise into the features) and then let the \agg{} function average out the injected noise (to an extent, not entirely), yielding a relatively good approximation of the neighborhood aggregation for the subsequent \upd{} function. The GNN's forward propagation can then proceed from this point without any modification to predict a class label for each node.

However, maintaining a proper balance between the accuracy of the GNN and the privacy of data introduces new challenges that need to be carefully addressed. On one hand, the node features to be collected are likely high-dimensional, so the perturbation of every single feature consumes a lot of the privacy budget. Suppose we want to keep our total budget $\epsilon$ low to provide better privacy protection. In that case, we need to perturb each of the $d$ features with $\epsilon/d$ budget (because the privacy budgets of the features add up together as the result of the composition theorem~\cite{dwork2014algorithmic}), which in turn results in adding more noise to the data that can significantly degrade the final accuracy. On the other hand, for the GNN to be able to cancel out the injected noise, the first layer's aggregator function must: (i) be in the form of a linear summation, and (ii) be calculated over a sufficiently large set of node features. However, not every GNN architecture employs a linear aggregator function, nor every node in the graph has many neighbors. In fact, in many real-world graphs that follow a Power-Law degree distribution, the number of low-degree nodes is much higher than the high-degree ones. Consequently, the estimated aggregation would most likely be very noisy, again leading to degraded performance.

To tackle the first challenge, we develop a multidimensional LDP method, called the \emph{multi-bit mechanism}, by extending the 1-bit mechanism~\cite{ding2017collecting} for multidimensional feature collection. It is composed of a user-side encoder and a server-side rectifier designed for maximum communication efficiency.
To address the second challenge, we propose a simple, yet effective graph convolution layer, called \emph{KProp}, which aggregates messages from an expanded neighborhood set that includes both the immediate neighbors and those nodes that are up to $K\text{-hops}$ away. By prepending this layer to the GNN, we can both combine our method with any GNN architecture and at the same time increase the graph convolution's estimation accuracy for low-degree nodes. In the experiments, we show that this technique can significantly boost the performance of our locally private GNN, especially for graphs with a lower average degree.

Finally, since the node labels are also considered private, we need another LDP mechanism to collect them privately. To this end, we use the generalized randomized response algorithm \cite{kairouz2016discrete}, which randomly flips the correct label to another one with a probability that depends on the privacy budget. However, learning the GNN with perturbed labels brings forward significant challenges in both training and validation. Regarding the former, training the GNN directly with the perturbed labels causes the model to overfit the noisy labels, leading to poor generalization performance. Regarding the latter, while it would be easy to validate the trained model using clean (non-perturbed) data, due to the privacy constraints of our problem, a more realistic setting is to assume that the server does not have access to any clean validation data. In this case, it is not clear how to perform model validation with noisy data, which is vital to prevent overfitting and optimize model hyper-parameters.

Although deep learning with noisy labels has been studied extensively in the literature \cite{song2020learning, patrini2017making, yi2019probabilistic, zhang2017mixup, zhang2018generalized, li2021unified, nt2019learning}, almost all the previous works either need clean features for training, require clean data for validation, or have been proposed for standard deep neural networks and do not consider the graph structure. Here, we propose \emph{Label \textbf{D}enoising with P\textbf{rop}agation -- Drop}, which incorporates the graph structure for label correction, and at the same time does not rely on any form of clean data (features or labels), neither for training nor validation. Given that nodes with similar labels tend to connect together more often \cite{wang2020unifying}, we utilize the graph topology to predict the label of a node by estimating the label frequency of its neighboring nodes. Still, if we rely on immediate neighbors, the true labels could not be accurately estimated due to insufficient neighbors for many nodes. Again, our key idea is to exploit KProp's denoising capability, but this time on node labels, to estimate the label frequency for each node and recovering the true label by choosing the most frequent one. Drop can easily be combined by any GNN architecture, and we show that it outperforms traditional baselines, especially at high-privacy regimes.

In the rest of this section, we describe our multi-bit mechanism, the KProp layer, and the Drop algorithms in more detail. The overview of our framework is depicted in Figure~\ref{fig:overview}. Note that the data perturbation step on the user-side has to be done only once for each node. The server collects the perturbed data once and stores it to train the GNN with minimum communication overhead.

\begin{figure*}
	\centering
	\includegraphics[width=0.7\textwidth]{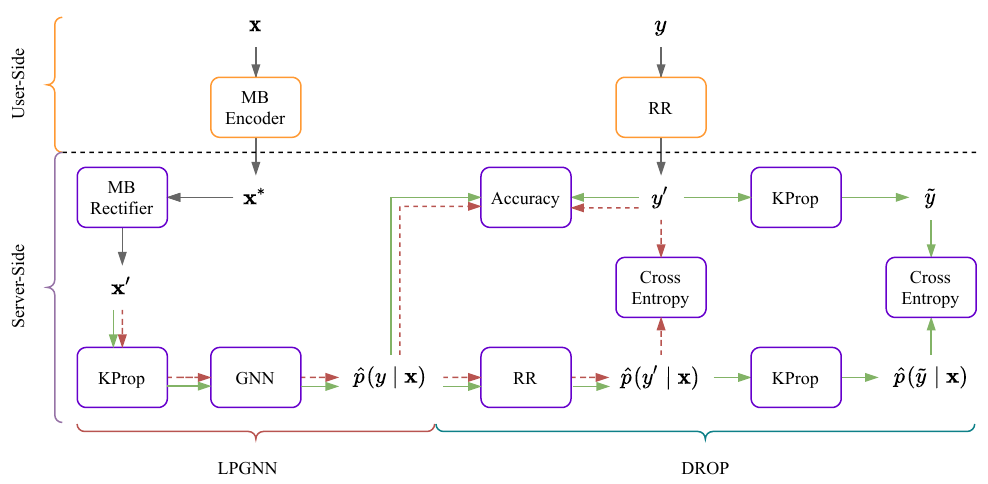}
	\caption{Overview of our locally private GNN training framework, featuring the multi-bit mechanism (MB Encoder and MB Rectifier), randomized response (RR), KProp layers, and Drop training. Users run multi-bit encoder and randomized response on their private features and labels, respectively, and send the output to the server, after which training begins. Green solid arrows and red dashed arrows indicate the training and validation paths, respectively.}
	\label{fig:overview}
\end{figure*}

\begin{algorithm}[t]
	\small
	\caption{Multi-Bit Encoder}\label{alg:mbm}
	\SetKwInOut{Input}{Input}
	\SetKwInOut{Output}{Output}
	\LinesNumbered
	\DontPrintSemicolon
	\Input{feature vector $\mathbf{x} \in [\alpha,\beta]^d$; privacy budget $\epsilon > 0$; range parameters $\alpha$ and $\beta$; sampling parameter $m\in\{1,2,\dots,d\}$.}
	\Output{encoded vector $\mathbf{x}^* \in \{-1,0,1\}^d.$}
	Let $\mathcal{S}$ be a set of $m$ values drawn uniformly at random without replacement from $\{1, 2, \ldots, d\}$\;
	\For{$i\in\{1, 2, \ldots, d\}$}
	{
		$s_i = 1$\text{ if }$i\in\mathcal{S}\text{ otherwise }s_i = 0$\;
		$t_i \sim \text{Bernoulli}\left( \frac{1}{e^\ek + 1} + \frac{x_{i} - \alpha}{\beta - \alpha}\cdot\frac{e^\ek - 1}{e^\ek + 1} \right)$\;
		$x^*_i = s_i\cdot(2t_i - 1)$\;
	}
	\Return $\mathbf{x}^* = [x^*_1,\dots,x^*_d]^T$\;
\end{algorithm}

\subsection{Collection of node features}
In this section, we explain our multi-bit mechanism for multidimensional feature perturbation, which is composed of an encoder and a rectifier, as described in the following.

\paragraph{\textbf{Multi-bit Encoder.}}
This part, which is executed at the user-side, perturbs the node's private feature vector and encodes it into a compact vector that is sent efficiently to the server.
More specifically, assume that every node $v$ owns a private $d\text{-dimensional}$ feature vector $\mathbf{x}_v$, whose elements lie in the range $[\alpha, \beta]$. When the server requests the feature vector of $v$, the node locally applies the multi-bit encoder on $\mathbf{x}_v$ to get the corresponding encoded feature vector $\mathbf{x}^*_v$, which is then sent back to the server. Since this process is supposed to be run only once, the generated $\mathbf{x}^*_v$ is recorded by the node to be returned in any subsequent calls to prevent the server from recovering the private feature vector using repeated queries.

Our multi-bit encoder is built upon the 1-bit mechanism \cite{ding2017collecting}, which returns either 0 or 1 for a single-dimensional input. However, as mentioned earlier, perturbing all the dimensions with a high-dimensional input results in injecting too much noise, as the total privacy budget has to be shared among all the dimensions. To balance the privacy-accuracy trade-off, we need to reduce dimensionality to decrease the number of dimensions that have to be perturbed. Still, since we cannot have the feature vectors of all the nodes at one place (due to privacy reasons), we cannot use conventional approaches, such as principal component analysis (PCA) or any other machine learning-based feature selection method. Instead, we randomly perturb a subset of the dimensions and then optimize the size of this subset to achieve the lowest variance in estimating the \agg{} function.

Algorithm~\ref{alg:mbm} describes this encoding process in greater detail. Intuitively, the encoder first uniformly samples $m$ out of $d$ dimensions without replacement, where $m$ is a parameter controlling how many dimensions are perturbed. Then, for each sampled dimension, the corresponding feature is randomly mapped to either -1 or 1, with a probability depending on the per-dimension privacy budget $\epsilon/m$ and the position of the feature value in the feature space, such that values closer to $\alpha$ (resp. $\beta$) are likely to be mapped to -1 (resp. 1). For other dimensions that are not sampled, the algorithm outputs 0. Therefore, a maximum of two bits per feature is enough to send $\mathbf{x}^*_v$ to the server. When $m=d$, our algorithm reduces to the 1-bit mechanism with a privacy budget of $\epsilon/d$ for every single dimension. The following theorem ensures that the multi-bit encoder is $\epsilon\text{-LDP}$ (proof in the Appendix).
\begin{theorem}\label{th:dp}
	The multi-bit encoder presented in Algorithm~\ref{alg:mbm} satisfies $\epsilon\text{-local differential}$ privacy for each node.
\end{theorem}

\paragraph{\textbf{Multi-bit Rectifier.}}
The output of the multi-bit encoder is statistically biased, i.e., $\E\left[\mathbf{x}^*\right]\ne\mathbf{x}$. Therefore, the goal of the multi-bit rectifier, executed at server-side, is to convert the encoded vector $\mathbf{x}^*$ to an unbiased perturbed vector $\mathbf{x}^\prime$, such that $\E\left[\mathbf{x}^\prime\right]=\mathbf{x}$, as follows:
\begin{equation}
	\mathbf{x}^{\prime} = Rect(\mathbf{x}^*) = \frac{d(\beta-\alpha)}{2m}\cdot\frac{e^\ek + 1}{e^\ek - 1}\cdot\mathbf{x}^* + \frac{\alpha + \beta}{2} \label{eq:est}
\end{equation}
Note that this is not a denoising process to remove the noise from $\mathbf{x}^*$, but the output vector $\mathbf{x}^\prime$ is still noisy and does not have any meaningful information about the private vector $\mathbf{x}$. The only difference between $\mathbf{x}^*$ and $\mathbf{x}^\prime$ is that the latter is unbiased, while the former is not. The following results entail from the multi-bit rectifier:

\begin{proposition}\label{prop:unbiased}
	The multi-bit rectifier defined by (\ref{eq:est}) is unbiased.
\end{proposition}

\begin{proposition}\label{prop:var}
	For any node $v$ and any $i\in\{1,2,\dots,d\}$, the variance of the multi-bit rectifier defined by (\ref{eq:est}) at dimension $i$ is:
	\begin{equation}\label{eq:var}
		Var[{x^\prime_{v,i}}] = \frac{d}{m}\cdot\left(\frac{\beta-\alpha}{2}\cdot\frac{e^\ek + 1}{e^\ek - 1}\right)^2 - \left({x}_{v,i} - \frac{\alpha + \beta}{2}\right)^2
	\end{equation}
\end{proposition}

The variance of an LDP mechanism is a key factor affecting the estimation accuracy: a lower variance usually leads to a more accurate estimation. Therefore, we exploit the result of Proposition~\ref{prop:var} to find the optimal sampling parameter $m$ in the multi-bit encoder (Algorithm~\ref{alg:mbm}) that minimizes the rectifier's variance, as follows:
\begin{proposition}\label{prop:optm}
	The optimal value of the sampling parameter $m$ in Algorithm~\ref{alg:mbm}, denoted by $m^\star$, is obtained as:
	\begin{equation}
		m^\star = \max(1, \min(d, \, \left\lfloor \frac{\epsilon}{2.18}\right\rfloor))
	\end{equation}
\end{proposition}

The above proposition implies that in the high-privacy regime $\epsilon \le 2.18$, the multi-bit mechanism perturbs only one random dimension. Therefore, this process is similar to a randomized one-hot encoding, except that here, the aggregation of these one-hot encoded features approximates the aggregation of the raw features.

\subsection{Approximation of graph convolution}
Upon collecting the encoded vectors $\mathbf{x}^*_v$ from every node $v$ and generating the corresponding perturbed vectors $\mathbf{x}^\prime_v$ using the multi-bit rectifier, the server can initiate the training of the GNN. In the first layer, the embedding for an arbitrary node $v$ is generated by the following (layer indicator subscripts and superscripts are omitted for simplicity):
\begin{align}
	\widehat{\mathbf{h}}_{\N(v)} & = \agg{}\left(\{\mathbf{x}_u^{\prime}, \forall u \in \mathcal{N}(v)\}\right) \label{eq:h} \\
	\mathbf{h}_{v}               & = \upd{}\left({\widehat{\mathbf{h}}_{\N(v)}}\right)
\end{align}
where $\widehat{\mathbf{h}}_{\N(v)}$ is the estimation of the first layer \agg{} function of any node $v$ by aggregating perturbed vectors $\mathbf{x}^\prime_u$ of all the nodes $u$ adjacent to $v$. After this step, the server can proceed with the rest of the layers to complete the forward propagation of the model, exactly similar to a standard GNN. If the \agg{} function is linear on its input (e.g., it is a weighted summation of the input vectors), the resulting aggregation would also be unbiased, as stated below:

\begin{corollary}\label{cor:unbiased}
	Given a linear aggregator function, the aggregation defined by (\ref{eq:h}) is an unbiased estimation for (\ref{eq:agg}) at layer $l=1$.
\end{corollary}

The following proposition also shows the relationship of the estimation error in calculating the \agg{} function and the neighborhood size $|\N(v)|$ for the special case of using the mean aggregator function:
\begin{proposition}\label{prop:error}
	Given the mean aggregator function for the first layer and $\delta > 0$, with probability at least $1-\delta$, for any node $v$, we have:
	\begin{equation}\label{eq:err}
		\max_{i\in\{1,\dots,d\}}\left|(\widehat{\mathbf{h}}_{\N(v)})_i - ({\mathbf{h}}_{\N(v)})_i\right| = \mathcal{O}\left(\frac{\sqrt{d \log(d/\delta)}}{\epsilon \sqrt{|\N(v)|}}\right)
	\end{equation}
\end{proposition}

The above proposition indicates that with the mean aggregator function (which can be extended to other \agg{} functions as well), the estimation error decreases with a rate proportional to the square root of the node's degree. Therefore, the higher number of neighbors, the lower the estimation error.
But as mentioned earlier, the size of $\N(v)$ is usually small in real graphs, which hinders the \agg{} function from driving out the injected noise on its own.

In a different context, prior works have shown that considering higher-order neighbors can help learn better node representations~\cite{morris2019weisfeiler, klicpera2019diffusion, pmlr-v97-abu-el-haija19a}. Inspired by these works, a potential solution to this issue is to expand the neighborhood of each node $v$ by considering more nodes that are not necessarily adjacent to $v$ but reside within an adjustable local neighborhood around $v$.
To this end, we use an efficient convolution layer, described in Algorithm~\ref{alg:kprop}, that can effectively be used to address the small-size neighborhood issue. The idea is simple: we aggregate features of those nodes that are up to $K$ steps away from $v$ by simply invoking the \agg{} function $K$ consecutive times, without any non-linear transformation in between. For simplicity, we call this algorithm KProp, as every node propagates its message to $K$ hops further.

As illustrated in Figure~\ref{fig:overview}, we prepend KProp as a denoising layer to the GNN. This approach has two advantages: first, it allows to use any GNN architecture with any \agg{} function for the backbone model, as KProp already uses a linear \agg{} that satisfies Corollary~\ref{cor:unbiased}; and second, it enables us to expand the effective aggregation set size for every node by controlling the step parameter $K$. However, it is essential to note that we cannot arbitrarily increase the neighborhood size around a node, since aggregating messages from too distant nodes could lead to over-smoothing of output vectors~\cite{li2018deeper}. Therefore, there is a trade-off between the KProp's denoising accuracy and the overall GNN's expressive power.

\begin{algorithm}[t]
	\small
	\SetKwInOut{Input}{Input}\SetKwInOut{Output}{Output}
	\LinesNumbered
	\Input{Graph $\mathcal{G}=(\mathcal{V}, \mathcal{E})$; input vector $\mathbf{x}_v, \forall v\in\mathcal{V}$; linear aggregator function \agg{}; step parameter $K\ge0$;
	}
	\Output{Embedding vector $\mathbf{h}_v, \forall v\in\mathcal{V}$}
	\DontPrintSemicolon
	\SetKwFor{ForPar}{for all}{do in parallel}{end}
	\caption{KProp Layer}
	\label{alg:kprop}

	\ForPar{$v \in \mathcal{V}$}{
	$\mathbf{h}^0_{\N(v)} = \mathbf{x}_v$\;
	\For{$k=1$ to $K$} {
	$\mathbf{h}^k_{\N(v)}= \agg{}\left(\{\mathbf{h}^{k-1}_{\N(u)}, \forall u \in \N(v) - \{v\}\}\right)$\;
	}
	$\mathbf{h}_v=\mathbf{h}^K_{\N(v)}$\;
	}
	\Return $\{\mathbf{h}_v, \forall v\in\mathcal{V}\}$\;
\end{algorithm}

It is worth mentioning that in KProp, we perform aggregations over $\mathcal{N}(v) - \{v\}$, i.e., we do not include self-loops. While it has been shown that adding self-loops can improve accuracy in conventional GNNs \cite{kipf2017semi}, excluding self-connections works better when dealing with noisy features. As $K$ grows, with self-loops, we account for the injected noise in the feature vector of each node in the $v$'s neighborhood multiple times in the aggregation. Therefore, removing self-loops helps to reduce the total noise by discarding repetitive node features from the aggregation.

\subsection{Learning with private labels}

In this last part, we describe the method used to perturb and collect labels privately and introduce our training algorithm for learning locally private GNNs using perturbed labels, called label denoising with propagation (Drop). Let $f(\mathbf{x}) = \arg\max_\mathbf{y}\hat p(\mathbf{y}\mid\mathbf{x})$ be the target node classifier, where $\hat p(\mathbf{y}\mid\mathbf{x}) = g(\mathbf{x},\mathcal{G};\mathbf{W})$ approximates the class-conditional probabilities $p(\mathbf{y}\mid\mathbf{x})$ and is modeled by a GNN $g(.)$ with the learnable weight matrix $\mathbf{W}$.
The goal is to optimize $\mathbf{W}$ such that $\hat p(\mathbf{y}\mid\mathbf{x})$ becomes as close as possible to $p(\mathbf{y}\mid\mathbf{x})$. In the standard setting, this is usually done by minimizing the cross-entropy loss function between $\hat p(\mathbf{y}\mid\mathbf{x})$ and true label $\mathbf{y}$ over the set of labeled nodes $\mathcal{V_L}$:
\begin{equation}
	\ell\left(\mathbf{y}, \hat p(\mathbf{y}\mid\mathbf{x})\right)=-\sum_{v\in\mathcal{V_L}}\mathbf{y}_v^T\log\hat p(\mathbf{y}\mid\mathbf{x_v})
\end{equation}

However, since the labels are considered private, each node $v\in\mathcal{V_L}$ that participates in the training procedure has to perturb their label $\mathbf{y}_v$ using some LDP mechanism, and send the perturbed label $\mathbf{y}^\prime_v$ to the server. Still, if we train the GNN using the perturbed labels by minimizing the cross-entropy loss between $\hat p(\mathbf{y}\mid\mathbf{x})$ and perturbed labels $\mathbf{y}^\prime$, namely $	\ell\left(\mathbf{y}^\prime, \hat p(\mathbf{y}\mid\mathbf{x})\right)$, the model completely overfits the noisy labels and generalizes poorly to unseen nodes. However, many real-world graphs, such as social networks, are homophilic \cite{mcpherson2001birds}, meaning that  nodes with structural similarity tend to have similar labels \cite{wang2020unifying}. We exploit this fact to estimate the frequency of the labels in a local neighborhood around any node $v$ to obtain its estimated label $\tilde{\mathbf{y}}_v$.
To this end, we can use any LDP frequency oracle, such as randomized response \cite{kairouz2016discrete}, Unary Encoding \cite{wang2017locally}, or Local Hashing \cite{wang2017locally}. In this paper, we use randomized response for two reasons: first, the number of classes is usually small, and randomized response has been shown to work better than other mechanisms in low dimensions \cite{wang2017locally}; and second, it introduces a symmetric, class-independent noise to the labels by flipping them according to the following distribution, which we later exploit in our learning algorithm:
\begin{equation}\label{eq:rr}
	p(\mathbf{y}^\prime \mid \mathbf{y}) = \begin{cases}
		\frac{e^\epsilon}{e^\epsilon + c - 1}, \quad \text{if}~\mathbf{y}^\prime=\mathbf{y} \\
		\frac{1}{e^\epsilon + c - 1}, \quad \text{otherwise}
	\end{cases}
\end{equation}
where $\mathbf{y}$ and $\mathbf{y}^\prime$ are clean and perturbed labels, respectively, $c$ is the number of classes, and $\epsilon$ is the privacy budget.

Similar to estimating the graph convolution with noisy features, we also face the problem of small-size neighborhood if we only rely on the first-order neighbors to estimate the label frequency. In order to expand the neighborhood around each node, we take the same approach as we did for features: we apply KProp on node labels, i.e., we set $\tilde{y}_v = \arg\max_{i\in[c]} h_i(\mathbf{y}^\prime_v, K_y)$ for all $v\in\mathcal{V_L}$, where $h(.)$ is the KProp function, $K_y$ is the step parameter, and $[c] = \{1,\dots,c\}$. With the mean aggregator function, at every iteration, KProp updates every node's label distribution by averaging its neighbors' label distribution. In this paper, however, we instead use the GCN aggregator function \cite{kipf2017semi}:
\begin{equation}\label{eq:gcn}
	\agg{}\left(\{\mathbf{y}_u^{\prime}, \forall u \in \mathcal{N}(v)\}\right)=\sum_{u\in\N(v)}\frac{\mathbf{y}^\prime_u}{\sqrt{|\N(u)|\cdot|\N(v)|}}
\end{equation}
Using the GCN aggregator leads to a lower estimation error than the mean aggregator due to the difference in their normalization factors, which affects their estimation variance. Specifically, the normalization factor in the GCN aggregator is $\sqrt{|\N(u)|\cdot|\N(v)|}$, while for the mean, it is $|\N(v)| = \sqrt{|\N(v)|\cdot|\N(v)|}$. In other words, the GCN aggregator considers the square root of the degree of both the central node $v$ and its neighbor $u$, whereas the mean aggregator considers only the square root of the central node $v$'s degree twice. Since there are many more low-degree nodes in many real graphs than high-degree ones, using the mean aggregator results in a small normalization factor for most nodes, leading to a higher estimation variance. But as many of the low-degree nodes are linked to the high-degree ones, the GCN aggregator balances the normalization by considering the degree of both link endpoints. Consequently, the normalization for many low-degree nodes increases compared to the mean aggregator, yielding a lower estimation variance.

As the step parameter $K_y$ gradually increases, the estimated label $\tilde{\mathbf{y}}$ becomes more similar to the clean label $\mathbf{y}$. Therefore, an initial idea for the training algorithm would be to learn the GNN using $\tilde{\mathbf{y}}$ instead of ${\mathbf{y}^\prime}$ by minimizing the cross-entropy loss between $\hat p(\mathbf{y}\mid\mathbf{x})$ and $\tilde{\mathbf{y}}$, namely $\ell\left( \tilde{\mathbf{y}}, \hat p(\mathbf{y}\mid\mathbf{x}) \right)$. However, this approach has two downsides. First, it causes the GNN to become a predictor for  $\tilde{\mathbf{y}}$ and not  ${\mathbf{y}}$. Although $\tilde{\mathbf{y}}$ tend to converge to ${\mathbf{y}}$ as $K_y$ increases, the output of KProp also becomes increasingly smoother, until the excessive KProp aggregations lead to over-smoothing, after which $\tilde{\mathbf{y}}$ will begin to diverge from ${\mathbf{y}}$ and become noisy again, while we are still fitting $\tilde{\mathbf{y}}$. Second, we cannot know how far we should increase $K_y$ to get the best accuracy and prevent over-smoothing without clean validation data. One way to validate the model with noisy labels is to calculate the accuracy of the target classifier $f(\mathbf{x})$ for predicting the estimated label $\tilde{\mathbf{y}}$. However, suppose the model overfits the over-smoothed labels. In that case, the corresponding validation $\tilde{\mathbf{y}}$'s also becomes over-smoothed and can be well predicted by the model, resulting in a high validation but low test accuracy.

To address the first issue, instead of minimizing $\ell\left( \tilde{\mathbf{y}}, \hat p(\mathbf{y}\mid\mathbf{x}) \right)$, we propose to minimize $\ell\left( \tilde{\mathbf{y}}, \hat p(\tilde{\mathbf{y}}\mid\mathbf{x}) \right)$, i.e., the cross-entropy loss between the estimated label $\tilde{\mathbf{y}}$ and its approximated probability $\hat p(\tilde{\mathbf{y}}\mid\mathbf{x})$, which can be obtained by applying the same procedure on $\hat p(\mathbf{y}\mid\mathbf{x})$ as we did on ${\mathbf{y}}$ to obtain $\tilde{\mathbf{y}}$. In the first place, we applied randomized response on ${\mathbf{y}}$ to obtain ${\mathbf{y}}^\prime$, and then passed the result to the KProp layer to get $\tilde{\mathbf{y}}$. If we go through the same steps to obtain $\hat p(\tilde{\mathbf{y}}\mid\mathbf{x})$ and then minimize its cross-entropy loss with $\tilde{\mathbf{y}}$, we can keep $\hat p({\mathbf{y}}\mid\mathbf{x})$ intact when KProp causes over-smoothing, and at the same time benefit from it's denoising capability. To this end, we first need to calculate $\hat p({\mathbf{y}}^\prime\mid\mathbf{x})$ from $\hat p({\mathbf{y}}\mid\mathbf{x})$:
\begin{equation}\label{eq:p_yp}
	\hat p({\mathbf{y}}^\prime\mid\mathbf{x}) = \sum_{\mathbf{y}} p({\mathbf{y}}^\prime\mid\mathbf{y})\cdot \hat p({\mathbf{y}}\mid\mathbf{x})
\end{equation}
where $p({\mathbf{y}}^\prime\mid\mathbf{y})$ is directly obtained from (\ref{eq:rr}). This step would be analogous to applying randomized response to ${\mathbf{y}}$ and getting ${\mathbf{y}}^\prime$. Finally, similar to applying KProp on ${\mathbf{y}}^\prime$ to get ${\tilde{\mathbf{y}}}$, we treat $\hat p({\mathbf{y}}^\prime\mid\mathbf{x})$ as soft labels and apply KProp with the same step parameter to approximate $\hat p(\tilde{\mathbf{y}}\mid\mathbf{x})$:
\begin{equation}\label{eq:p_yt}
	\hat p(\tilde{\mathbf{y}}\mid\mathbf{x}) = softmax\left(h\left(\hat p({\mathbf{y}}^\prime\mid\mathbf{x}), K_y\right)\right)
\end{equation}
where the softmax is used to normalize the KProp's output as a valid probability distribution. Finally, we train the model by minimizing $\ell\left( \tilde{\mathbf{y}}, \hat p(\tilde{\mathbf{y}}\mid\mathbf{x}) \right)$.

To address the validation issue, we must make sure that our validation procedure is not affected by the KProp step parameter $K_y$. Clearly, if we use $\tilde{\mathbf{y}}$ for validation, by changing $K_y$ we are also modifying estimated labels $\tilde{\mathbf{y}}$, and thus we are basically validating different models with different labels. Therefore, we should only validate the model using the noisy labels ${\mathbf{y}}^\prime$. Here, we choose the cross-entropy loss between ${\mathbf{y}}^\prime$ and $\hat p({\mathbf{y}}^\prime\mid\mathbf{x})$, namely $\ell\left( {\mathbf{y}}^\prime, \hat p({\mathbf{y}}^\prime\mid\mathbf{x}) \right)$, as Patrini \etal~\cite{patrini2017making} show that this loss function, which they call forward correction loss, is unbiased, meaning that under expected label noise, $\ell\left( {\mathbf{y}}^\prime, \hat p({\mathbf{y}}^\prime\mid\mathbf{x}) \right)$ is equal to $\ell\left( {\mathbf{y}}, \hat p({\mathbf{y}}\mid\mathbf{x}) \right)$, i.e., the original loss computed on clean data. Therefore, we train the GNN with different hyper-parameters, including $K_y$, and pick the one achieving the lowest forward correction loss.

While this is in principle a reasonable idea, the forward correction loss on its own is not enough to prevent overfitting. That's because when $K_y$ is small, the estimated label $\tilde{\mathbf{y}}$ is more similar to the noisy one ${\mathbf{y}}^\prime$ than the clean label ${\mathbf{y}}$, and thus the model overfits to the noisy labels. In this case, the GNN becomes a predictor for ${\mathbf{y}}^\prime$ rather than ${\mathbf{y}}$, yielding a small forward correction loss. This is an incorrect validation signal as it favors $K_y$ to be close to zero. To overcome this issue and detect overfitting to label noise, our approach is to look instead at the accuracy of the target classifier $f(\mathbf{x}) = \arg\max_\mathbf{y}\hat p(\mathbf{y}\mid\mathbf{x})$ for predicting the noisy labels $\mathbf{y}^\prime$. From the randomized response algorithm, we know that the probability of keeping the label is $\frac{e^\epsilon}{e^\epsilon + c - 1}$. This gives us an upper bound on the expected accuracy of a perfect classifier, i.e., the classifier with 100\% accuracy on predicting clean labels. In other words, a perfect classifier can predict ${\mathbf{y}}^\prime$ with an expected accuracy of at most $Acc^* = \frac{e^\epsilon}{e^\epsilon + c - 1}$. Therefore, if during training the model we get accuracy above $Acc^*$, either on the training or validation dataset, we can consider it a signal of overfitting to the noisy labels. More specifically, we train the GNN for a maximum of $T$ epochs and record the forward correction loss and the accuracy of the target classifier for predicting noisy labels over both training and validation sets at every epoch. At the end of training, we pick the model achieving the lowest forward correction loss such that their accuracy is at most $Acc^*$.

Putting all together, the pseudo-code of the LPGNN training algorithm with Drop is presented in Algorithm~\ref{alg:dpgnn}, where we use two different privacy budgets $\epsilon_x$ and $\epsilon_y$ for feature and label perturbation, respectively. The following corollary entails from our algorithm:
\begin{corollary}\label{cor:dp}
	Algorithm~\ref{alg:dpgnn} satisfies $(\epsilon_x+\epsilon_y)\text{-local}$ differential privacy for graph nodes.
\end{corollary}

Corollary~\ref{cor:dp} shows that the entire training procedure is LDP due to the robustness of differential privacy to post-processing~\cite{dwork2014algorithmic}. Furthermore, any prediction performed by the LPGNN is again subject to the post-processing theorem~\cite{dwork2014algorithmic}, and therefore, satisfies LDP for the nodes, as the LDP mechanism is applied to the private data only once.

\begin{algorithm}[t]
	\small
	\SetKwInOut{Input}{Input}\SetKwInOut{Output}{Output}
	\SetKwRepeat{Do}{do}{while}%
	\SetKwFor{ForPar}{for all}{do in parallel}{end}
	\LinesNumbered
	\Input{Graph $\mathcal{G}=(\mathcal{V_L}\cup\mathcal{V_U}, \mathcal{E})$; GNN model $g(\mathbf{x},\mathcal{G};\mathbf{W})$; KProp layer $h(\mathbf{x},\mathcal{G}; K)$; KProp step parameter for features $K_x\ge0$; KProp step parameter for labels $K_y\ge0$; privacy budget for feature perturbation $\epsilon_x > 0$; privacy budget for label perturbation $\epsilon_y > 0$; range parameters $\alpha$ and $\beta$; number of classes $c$; maximum number of epochs $T$; learning rate $\eta$;}
	\Output{Trained GNN weights $\mathbf{W}$}
	\BlankLine
	\DontPrintSemicolon
	\caption{Locally Private GNN Training with Drop}
	\label{alg:dpgnn}
	\textbf{Server-side:}\;
	$\mathcal{V}\gets\mathcal{V_L}\cup\mathcal{V_U}$\;
	Send $\epsilon_x$, $\epsilon_y$, $\alpha$, and $\beta$ to every node $v\in\mathcal{V}$.\;
	\BlankLine
	\textbf{Node-side:}\;
	Obtain a perturbed vector $\mathbf{x}^*$ by Algorithm~\ref{alg:mbm}.\;
	\eIf{current node is in $\mathcal{V_L}$}{
		Obtain a perturbed label $\mathbf{y}^\prime$ by (\ref{eq:rr}).\;
	}{
		$\mathbf{y}^\prime \gets \vec{\mathbf{0}}$\;
	}
	Send $(\mathbf{x}^*, \mathbf{y}^\prime)$ to the server.\;
	\BlankLine
	\textbf{Server-side:}\;
	Obtain $\mathbf{x}^\prime_v$ using (\ref{eq:est}) for all $v\in\mathcal{V}$.\;
	$\mathbf{h}_v\gets h(\mathbf{x}^\prime_v,\mathcal{G}; K_x)$ for all $v\in\mathcal{V}$.\;
	$\tilde{\mathbf{y}}_v\gets h(\mathbf{y}^\prime_v,\mathcal{G}; K_y)$ for all $v\in\mathcal{V_L}$.\;
	Partition $\mathcal{V_L}$ into train and validation sets $\mathcal{V_L}^{tr}$ and $\mathcal{V_L}^{val}$.\;
	$Acc^* \gets {e^{\epsilon_y}}/{(e^{\epsilon_y} + c - 1)}$\;
	\For{$t\in\{1,\dots,T\}$} {
	\ForPar{$v\in\mathcal{V_L}$}{
		$\hat p(\mathbf{y}\mid\mathbf{x}_v) \gets g(\mathbf{h}_v, \mathcal{G};\mathbf{W})$\;
		Obtain $\hat p(\mathbf{y}^\prime\mid\mathbf{x}_v)$ using (\ref{eq:p_yp})\;
		Obtain $\hat p(\tilde{\mathbf{y}}\mid\mathbf{x}_v)$ using (\ref{eq:p_yt})\;
	}
	$\mathbf{W}^{t+1}\gets\mathbf{W}^{t}-\eta\nabla\sum_{v\in\mathcal{V_L}^{tr}}\ell\left( \tilde{\mathbf{y}}_v, \hat p(\tilde{\mathbf{y}}\mid\mathbf{x}_v) \right)$\;
	$\ell^t_{val} \gets \sum_{v\in\mathcal{V_L}^{val}}\ell\left( {\mathbf{y}}_v^\prime, \hat p({\mathbf{y}}^\prime\mid\mathbf{x}_v) \right)$\;
	$Acc^t_{val} \gets\frac{1}{|\mathcal{V_L}^{val}|}\sum_{v\in\mathcal{V_L}^{val}} Accuracy(\hat p(\mathbf{y}\mid\mathbf{x}_v), {\mathbf{y}}_v^\prime)$\;
	$Acc^t_{tr} \gets \frac{1}{|\mathcal{V_L}^{tr}|}\sum_{v\in\mathcal{V_L}^{tr}} Accuracy(\hat p(\mathbf{y}\mid\mathbf{x}_v), {\mathbf{y}}_v^\prime)$\;
	}
	$t\gets\arg\min_t\ell^t_{val}$ such that $Acc^t_{tr} \le Acc^*$ and $Acc^t_{val} \le Acc^*$\;
	\Return  $\mathbf{W}^t$\;
\end{algorithm}

\section{Experiments}\label{sec:xp}

We conduct extensive experiments to assess the privacy-utility performance of the proposed method for the node classification task and evaluate it under different parameter settings that can affect its effectiveness.

\subsection{Experimental settings}

\paragraph{\textbf{Datasets.}}
We used two different sets of publicly available real-world datasets: two citation networks, Cora and Pubmed~\cite{yang2016revisiting}, which have a lower average degree, and two social networks, Facebook \cite{rozemberczki2019multi}, and LastFM \cite{feather} that have a higher average degree. The description of the datasets is as followed:
\begin{itemize}[leftmargin=*]
	\item{{\emph{Cora and Pubmed~\cite{yang2016revisiting}:}}} These are well-known citation network datasets, where each node represents a document and edges denote citation links. Each node has a bag-of-words feature vector and a label indicating its category.
	\item{{\emph{Facebook~\cite{rozemberczki2019multi}:}}} This dataset is a page-page graph of verified Facebook sites. Nodes are official Facebook pages, and edges correspond to mutual likes between them. Node features are extracted from the site descriptions, and the labels denote site category.
	\item{{\emph{LastFM~\cite{feather}:}}} This social network is collected from the music streaming service LastFM. Nodes denote users from Asian countries, and links correspond to friendships. The task is to predict the home country of a user given the artists liked by them. Since the original dataset was highly imbalanced, we limited the classes to the top-10 having the most samples.
\end{itemize}
Summary statistics of the datasets are provided in Table~\ref{tab:dataset}.

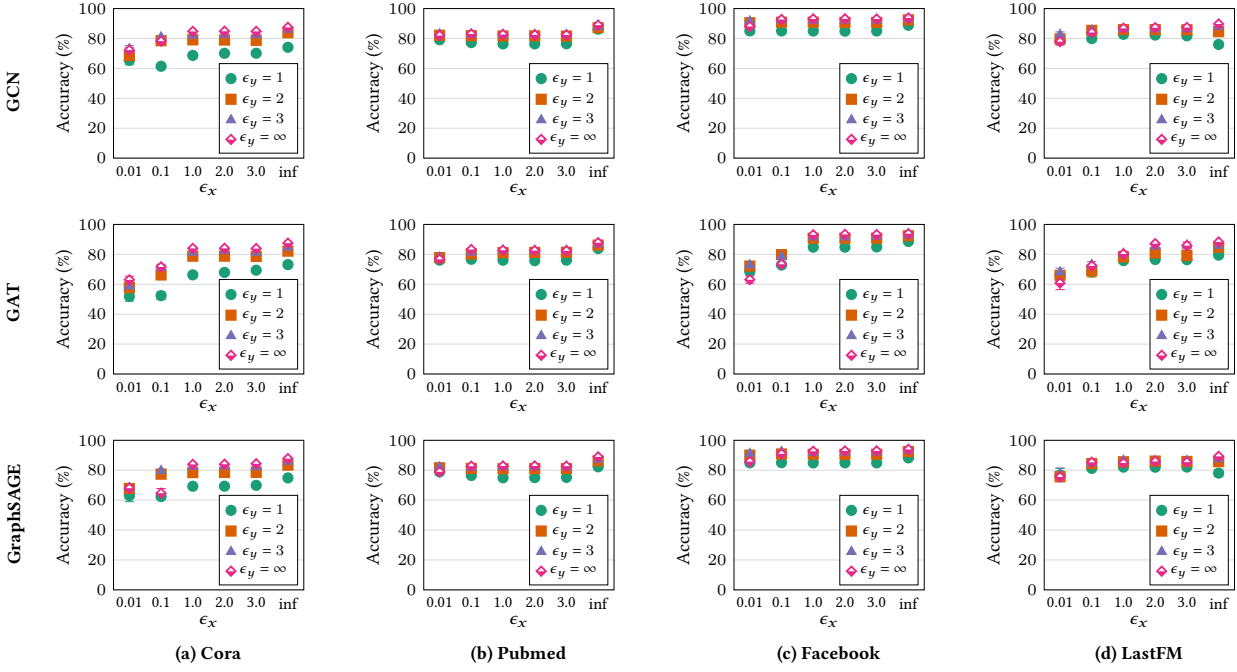
\begin{figure*}[t]
	\newcommand{\plot}[3][]{
		\resizebox{\boxsize}{!}{
			\begin{tikzpicture}[trim axis left, trim axis right, every mark/.append style={mark size=2pt}]
				\begin{groupplot}[
					group style={group size=1 by 3, vertical sep=1cm},
					footnotesize,
					width=\figwidth,
					legend pos=south east,
					legend style={
							nodes={scale=0.9, transform shape},
						},
					legend entries={$\epsilon_y=1$, $\epsilon_y=2$, $\epsilon_y=3$, $\epsilon_y=\infty$},
					xlabel=$\epsilon_x$,
					ylabel=Accuracy (\%),
					ylabel shift = -4 pt,
					xlabel shift = -3 pt,
					ymin=0,ymax=100,
					ymajorgrids,
					symbolic x coords={0.01, 0.1, 1.0, 2.0, 3.0, inf},
					xtick=={0.01, 0.1, 1.0, 2.0, 3.0, $\infty$},
					xtick style={draw=none},
					ytick style={draw=none},
					]
					\pgfplotsinvokeforeach{gcn, gat, sage} {
						\nextgroupplot[
							title=\textbf{\map{##1}},
							title style={at={(-0.6,0.4)},anchor=north,rotate=90, color=white},
							#1]
						\foreach \yeps in {1.0, 2.0, 3.0, inf} {
								\addplot+[thick, only marks, error bars/.cd, y dir=both, y explicit]
								table[x=x_eps, y=test/acc_mean, y error=test/acc_ci, col sep=comma] {figures/xyeps/model=##1,dataset=#3,y_eps=\yeps.csv};
							}
					}
				\end{groupplot}
			\end{tikzpicture}
		}
	}

	\centering
	\subfloat[Cora]{\plot[title style={color=black}]{}{cora}}\hfil
	\subfloat[Pubmed]{\plot{}{pubmed}}\hfil
	\subfloat[Facebook]{\plot{}{facebook}}\hfil
	\subfloat[LastFM]{\plot{}{lastfm}}
	\caption{Comparison of LPGNN's performance with different GNN model under varying feature and label privacy budgets.}
	\label{fig:eps}
\end{figure*}

\begin{table}[t]
	\centering
	\small
	\caption{Descriptive statistics of the used datasets}
	\label{tab:dataset}
	\sc
	\begin{tabu} to \columnwidth {X[l] X[c] X[c] X[c] X[c] c}
		\toprule
		Dataset  & \#Classes & \#Nodes & \#Edges & \#Features & Avg. Deg. \\
		\midrule
		Cora     & 7         & 2,708   & 5,278   & 1,433      & 3.90      \\
		Pubmed   & 3         & 19,717  & 44,324  & 500        & 4.50      \\
		Facebook & 4         & 22,470  & 170,912 & 4,714      & 15.21     \\
		LastFM   & 10        & 7,083   & 25,814  & 7,842      & 7.29      \\
		\bottomrule
	\end{tabu}
\end{table}

\paragraph{\textbf{Experiment setup.}}
For all the datasets, we randomly split nodes into training, validation, and test sets with 50/25/25\% ratios, respectively. Without loss of generality, we normalized the node features of all the datasets between zero and one\footnote{Note that this normalization step does not affect the privacy, as the range parameters $(\alpha, \beta)$ are known to both the server and users, so the server could ask users to normalize their data between 0 and 1 before applying the multi-bit encoder.}, so in all cases, we have ${\alpha}=0$ and ${\beta}=1$.
LDP feature perturbation is applied to the features of all the training, validation, and test sets. However, label perturbation is only applied to the training and validation sets, and the test set's labels are left clean for performance testing.
We tried three state-of-the-art GNN architectures, namely GCN \cite{kipf2017semi}, GAT \cite{velivckovic2017graph}, and GraphSAGE \cite{hamilton2017inductive}, as the backbone model for LPGNN, with GraphSAGE being the default model for ablation studies.
All the GNN models have two graph convolution layers with a hidden dimension of size 16 and the SeLU activation function~\cite{klambauer2017self} followed by dropout, and the GAT model has four attention heads. For both feature and label KProps, we use GCN aggregator function.
We optimized the hyper-parameters of LPGNN based on the validation loss of GraphSAGE using the Drop algorithm as described in Section~\ref{sec:method} with the following strategy, and used the same values for other GNN models: First, we fix $K_x$ and $K_y$ to $(16, 8)$, $(16, 2)$, $(4, 2)$, and $(8, 2)$, on Cora, Pubmed, Facebook, and  LastFM, respectively, and for every pair of privacy budgets $(\epsilon_x, \epsilon_y)$ in $(1,1)$, $(1,\infty)$, $(\infty,1)$, and $(\infty,\infty)$, we perform a grid search to find the best choices for initial learning rate and weight decay both from $\{10^{-4}, 10^{-3}, 10^{-2}\}$ and dropout rate from $\{10^{-4}, 10^{-3}, 10^{-2}\}$. Second, we fix the best found hyper-parameters in the previous step and search for the best performing KProp step parameters $K_x$ and $K_y$ both within $\{0,2,4,8,16\}$ for all $\epsilon_x\in\{0.01, 0.1, 1, 2, 3, \infty\}$ and $\epsilon_y\in\{0.5, 1, 2, 3, \infty\}$. More specifically, for every $\epsilon_x$ (resp. $\epsilon_y$) except $\infty$, we use the best learning rate, weight decay, and dropout rate found for $\epsilon_x=1$ in the previous step (resp. $\epsilon_y=1$) to search for the best KProp step parameters.
All the models are trained using the Adam optimizer~\cite{kingma2014adam} over a maximum of 500 epochs, and the best model is picked for testing based on the validation loss. We measured the accuracy on the test set over 10 consecutive runs and report the average and 95\% confidence interval calculated by bootstrapping with 1000 samples. Our implementation is available at \href{https://github.com/sisaman/LPGNN}{https://github.com/sisaman/LPGNN}.


\subsection{Experimental results}

\paragraph{\textbf{Analyzing the utility-privacy trade-off.}}
We first evaluate how our privacy-preserving LPGNN method performs under varying feature and label privacy budgets. We changed the feature privacy budget $\epsilon_x$ in $\{0.01, 0.1, 1, 2, 3, \infty\}$ and the label privacy budget within $\{1,2,3,\infty\}$. The cases where $\epsilon_x=\infty$ or $\epsilon_y=\infty$, are provided for comparison with non-private baselines, where we did not apply the corresponding LDP mechanism (multi-bit for features and randomized response for labels) and directly used the clean (non-perturbed) values. We performed this experiment using GCN, GAT, and GraphSAGE as different backbone GNN models and reported the node-classification accuracy, as illustrated in Figure~\ref{fig:eps}. 

We can observe that all the three GNN models demonstrate robustness to the perturbations, especially on features, and perform comparably to the non-private baselines. For instance, on the Cora dataset, both GCN and GraphSAGE could get an accuracy of about 80\%  at $\epsilon_x=0.1$ and $\epsilon_y=2$, which is only 6\% lower than the non-private ($\epsilon=\infty$) method. On the other three datasets, we can decrease $\epsilon_x$ to 0.01 and $\epsilon_y$ to 1, and still get less than 10\% accuracy loss compared to the non-private baseline. We believe that this is a very promising result, especially for a locally private model perturbing hundreds of features with a low privacy loss. This result shows that different components of LPGNN, from multi-bit mechanism to KProp, and the Drop algorithm are fitting well together.

According to the results, the GAT model slightly falls behind GCN and GraphSAGE in terms of accuracy-privacy trade-off, especially at high-privacy regimes $\epsilon_x\le1$, which is mainly due to its stronger dependence on the node features. Unlike the other two models, GAT uses node features at each layer to learn attention coefficients first, which are then used to weight different neighbors in the neighborhood aggregation. This property of GAT justifies its sensitivity to the features, and thus it degrades more than the other two models when the features are highly noisy. On the contrary, a model like GCN only uses node features in the GCN aggregator function (Eq.~\ref{eq:gcn}) and thus can better tolerate the noisy features. GraphSAGE averages neighboring node features and then appends the self feature vector to the aggregation, and therefore it is not as dependent as GAT on the features. Nevertheless, GAT could also achieve comparable results for $\epsilon_x\ge1$ on all the datasets.

\paragraph{\textbf{Analyzing the multi-bit mechanism.}}

In Table~\ref{tab:mechanisms}, we compared the performance of our multi-bit mechanism (denoted as MB) against 1-bit mechanism (1B), Laplace mechanism (LP), and Analytic Gaussian mechanism (AG) \cite{balle2018improving}. The 1-bit mechanism \cite{ding2017collecting}, is obtained by setting $m=d$ in Algorithm~\ref{alg:mbm}. The Laplace and Gaussian mechanisms are two classic mechanisms that respectively add a zero-mean Laplace and Gaussian noise to the data with a noise variance calibrated based on the privacy budget, and are widely used for both single value and multidimensional data perturbation. Note that the Gaussian mechanism satisfies a relaxed version of \ldp, namely $(\epsilon,\delta)\text{-LDP}$, which (loosely speaking) means that it satisfies \ldp with probability at least $1-\delta$ for $\delta > 0$. Here, we use the Analytic Gaussian mechanism~\cite{balle2018improving}, the optimized version of the standard Gaussian mechanism, with $\delta=1^{-10}$. As all these mechanisms are used for feature perturbation, we set the label privacy budget $\epsilon_y=\infty$ and only consider their performance under different $\epsilon_x\in\{0.01, 0.1, 1, 2\}$. According to the results, our multi-bit mechanism consistently outperforms the other mechanisms in classification accuracy almost in all cases, especially under smaller privacy budgets. For instance, at $\epsilon_x=0.01$, MB performs over 8\%, 2\%, 7\%, and 12\% better than the second-best mechanism AG on Cora, Pubmed, Facebook, and LastFM, respectively. This is mainly because the variance of our optimized multi-bit mechanism is lower than the other three, resulting in a more accurate estimation. Simultaneously, our mechanism is also efficient in terms of the communication overhead, requiring only two bits per feature. In contrast, the Gaussian mechanism's output is real-valued, usually taking 32 bits per feature (more or less, depending on the precision) to transmit a floating-point number.

To verify that using node features in a privacy-preserving manner has an added value in practice, in Table~\ref{tab:features}, we compare our multi-bit features with several ad-hoc feature vectors that can be used instead of the private features to train the GNN without any additional privacy cost.
\textsc{Ones} is the all-one feature vector, \textsc{Ohd} is the one-hot encoding of the node's degree, as in \cite{xu2018powerful}, and
\textsc{Rnd} is randomly initialized node features between 0 and 1.
To have a fair comparison, we set the feature dimension of all the methods
equal to the private features. We set $\epsilon_y=1$ and compare the LPGNN's result with multi-bit encoded features under $\epsilon_x\in\{0.01, 0.1, 1\}$. We observe that LPGNN, with the multi-bit mechanism, even under the minimum privacy budget of 0.01, significantly outperforms the ad-hoc baselines in all cases, with an improvement ranging from around 7\% on Facebook to over 20\% on Pubmed comparing to the best performing ad-hoc baseline. Note that even though perturbed features under very small $\epsilon_x$ are noisier and becomes similar to \textsc{Rnd}, with the help of KProp, the resulting aggregation could estimate -- even if poorly -- the true aggregation, which might be enough for the GNN to distinguish between different neighborhoods. But in the case of \textsc{Rnd}, the aggregations carry no information about neighborhoods as the features are random, so the accuracy is worse than the multi-bit mechanism. This result shows that node features are also effective in addition to the graph structure, and we cannot ignore their utility.

\begin{table}[t]
	\centering
	\caption{Accuracy of LPGNN with different \\ LDP mechanisms (${\epsilon_y=\infty}$)}
	\label{tab:mechanisms}
	\small
	\sc
	\begin{tabu} to \linewidth {X[l] X[c] c c c c}
		\toprule
		Dataset  & Mech. & $\epsilon_x=0.01$       & $\epsilon_x=0.1$        & $\epsilon_x=1$          & $\epsilon_x=2$          \\
		\midrule
		Cora     & 1b    & 45.8 $\pm$ 3.3          & 62.3 $\pm$ 1.5          & 59.9 $\pm$ 2.7          & 58.5 $\pm$ 2.9          \\
		         & lp    & 43.2 $\pm$ 3.1          & 57.8 $\pm$ 2.3          & 61.9 $\pm$ 3.1          & 58.1 $\pm$ 2.1          \\
		         & ag    & 59.7 $\pm$ 2.3          & 62.7 $\pm$ 2.8          & 67.5 $\pm$ 3.0          & 77.2 $\pm$ 1.9          \\
		         & mb    & \textbf{68.0 $\pm$ 2.9} & \textbf{64.6 $\pm$ 3.2} & \textbf{83.9 $\pm$ 0.4} & \textbf{84.0 $\pm$ 0.3} \\
		\midrule
		Pubmed   & 1b    & 76.2 $\pm$ 0.6          & 74.8 $\pm$ 0.7          & 81.8 $\pm$ 0.4          & 82.5 $\pm$ 0.2          \\
		         & lp    & 76.6 $\pm$ 0.5          & 75.2 $\pm$ 1.0          & 81.9 $\pm$ 0.4          & 82.4 $\pm$ 0.2          \\
		         & ag    & 76.4 $\pm$ 0.6          & 81.5 $\pm$ 0.3          & 82.9 $\pm$ 0.2          & \textbf{83.1 $\pm$ 0.2} \\
		         & mb    & \textbf{78.9 $\pm$ 0.7} & \textbf{82.7 $\pm$ 0.2} & \textbf{82.9 $\pm$ 0.2} & 82.9 $\pm$ 0.1          \\
		\midrule
		Facebook & 1b    & 57.0 $\pm$ 3.4          & 76.3 $\pm$ 1.6          & 86.1 $\pm$ 0.6          & 84.0 $\pm$ 1.3          \\
		         & lp    & 54.2 $\pm$ 2.9          & 72.5 $\pm$ 2.1          & 85.4 $\pm$ 0.4          & 84.8 $\pm$ 1.6          \\
		         & ag    & 78.2 $\pm$ 1.4          & 85.6 $\pm$ 0.7          & 92.0 $\pm$ 0.1          & 92.4 $\pm$ 0.2          \\
		         & mb    & \textbf{85.8 $\pm$ 0.4} & \textbf{91.0 $\pm$ 0.4} & \textbf{92.7 $\pm$ 0.1} & \textbf{92.9 $\pm$ 0.1} \\
		\midrule
		LastFm   & 1b    & 40.5 $\pm$ 7.4          & 56.2 $\pm$ 2.1          & 75.5 $\pm$ 2.5          & 68.1 $\pm$ 4.1          \\
		         & lp    & 43.4 $\pm$ 5.7          & 50.5 $\pm$ 2.7          & 73.1 $\pm$ 2.9          & 67.2 $\pm$ 6.7          \\
		         & ag    & 63.6 $\pm$ 2.4          & 75.1 $\pm$ 1.9          & 67.7 $\pm$ 4.2          & 63.5 $\pm$ 4.6          \\
		         & mb    & \textbf{75.6 $\pm$ 1.6} & \textbf{85.3 $\pm$ 0.4} & \textbf{84.9 $\pm$ 0.8} & \textbf{85.9 $\pm$ 1.1} \\
		\bottomrule
	\end{tabu}
\end{table}

\paragraph{\textbf{Analyzing the effect of KProp.}}
In this experiment, we investigate whether the KProp layer can effectively gain performance boost, for either node feature or labels. For this purpose, we varied the KProp's step parameters $K_x$ and $K_y$ both within $\{0,2,4,8,16\}$, and trained the LPGNN model under varying privacy budget, whose result is depicted in Figure~\ref{fig:kprop}. In the top row of the figure, we change $K_x\in\{0,2,4,8,16\}$ and $\epsilon_x\in\{0.01, 0.1, 1\}$, while fixing $\epsilon_y$=1 and selecting the best values for $K_y$ based on the validation loss. Conversely, in the bottom row, we vary $K_y\in\{0,2,4,8,16\}$ and $\epsilon_y\in\{0.5, 1, 2\}$, and set the best $K_x$ at $\epsilon_x=1$.

We observe that in all cases, both the feature and the label KProp layers are effective and can significantly boost the accuracy of the LPGNN depending on the dataset and the value of the corresponding privacy budget.
Based on the results, the accuracy of LPGNN rises to an extent by increasing the step parameters, which shows that the model can benefit from larger population sizes to have a better estimation for both graph convolution and labels. Furthermore, we see that the maximum performance gain is different across the datasets and privacy budgets. As the estimates become more accurate due to an increase in the privacy budget, we see that KProp becomes less effective, mainly due to over-smoothing. But at lower privacy budgets, KProp usually achieves the highest relative accuracy gain.

In the case of feature KProp, the performance is also correlated to the average node degree. For instance, at $\epsilon_x=0.01$, on the social network datasets with a higher average degree, the accuracy gain is around 6\% and 10\% on Facebook and LastFM, respectively, while on lower-degree citation networks, it is over 20\% on both Cora and Pubmed, which suggests that lower-degree datasets can benefit more from KProp. Furthermore, the optimal step parameter $K_x$ that yields the best result also depends on the average degree of the graph. For example, we see that the trend is more or less increasing until the end for citation networks with a lower average degree. In contrast, the accuracy begins to fall over higher-degree social networks after $K_x=4$. This means that in lower-degree datasets, KProp requires more steps to reach the sufficient number of nodes for aggregation, while on higher-degree graphs, it can achieve this number in fewer steps.

Regarding the label KProp, the performance growth depends not only on the average degree, but also on the number of classes, which can significantly affect the accuracy of randomized response. For instance, despite its high average degree, KProp could increase the accuracy on LastFM with 10 classes by over 20\% at $\epsilon_y=0.5$, while on the other high-degree dataset, Facebook, which has 4 classes, this number is at most 5\%. Low-degree datasets still can benefit much from label KProp, with both Cora and Pubmed achieving a maximum of 30\% accuracy boost at $\epsilon_y=1$ and $\epsilon_y=0.5$, respectively.

\begin{table}[t]
	\centering
	\caption{Accuracy of LPGNN with different features \\ ($\epsilon_y=1$)}
	\label{tab:features}
	\small
	\sc
	\begin{tabu} to \columnwidth {l X[c] X[c] X[c] X[c]}
		\toprule
		Feature                 & Cora           & Pubmed         & Facebook       & LastFm         \\
		\midrule
		Ones                    & 22.6 $\pm$ 5.0 & 38.9 $\pm$ 0.4 & 29.0 $\pm$ 1.4 & 19.6 $\pm$ 1.8 \\
		Ohd                     & 44.4 $\pm$ 3.5 & 52.5 $\pm$ 5.7 & 77.2 $\pm$ 0.3 & 66.4 $\pm$ 1.6 \\
		Rnd                     & 26.4 $\pm$ 3.0 & 56.0 $\pm$ 1.3 & 35.2 $\pm$ 5.6 & 32.3 $\pm$ 6.3 \\
		\midrule
		Mbm ($\epsilon_x=0.01$) & 63.0 $\pm$ 4.1 & 78.9 $\pm$ 0.2 & 85.0 $\pm$ 0.4 & 76.9 $\pm$ 4.3 \\
		Mbm ($\epsilon_x=0.1$)  & 62.4 $\pm$ 2.0 & 76.5 $\pm$ 0.4 & 85.1 $\pm$ 0.2 & 81.2 $\pm$ 1.3 \\
		Mbm ($\epsilon_x=1$)    & 69.3 $\pm$ 1.2 & 74.9 $\pm$ 0.3 & 84.9 $\pm$ 0.2 & 82.1 $\pm$ 1.0 \\
		\bottomrule
	\end{tabu}
\end{table}

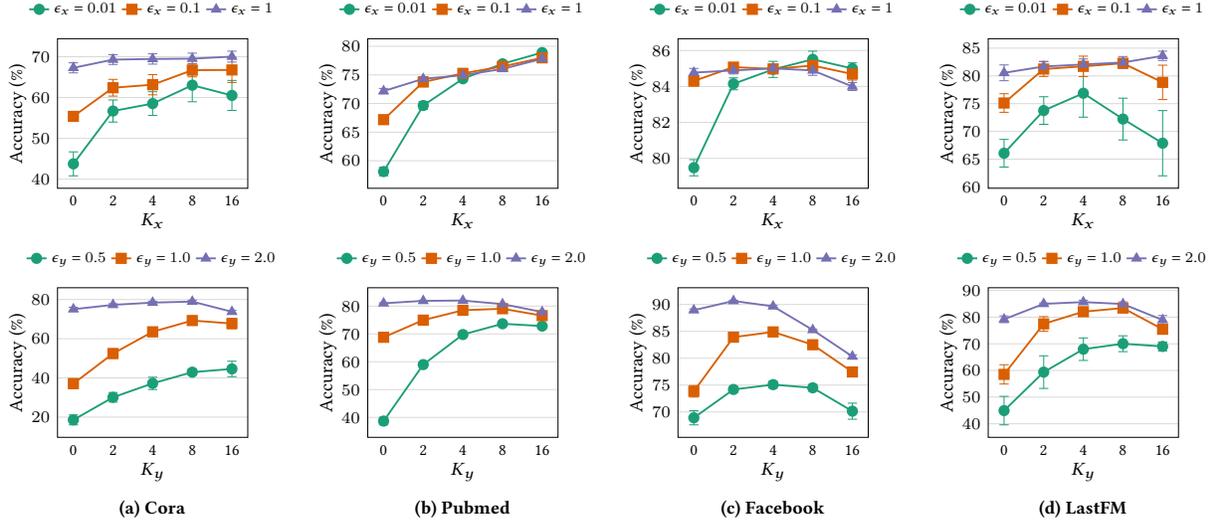
\begin{figure*}[t]
	\newcommand{\plot}[3][]{
		\resizebox{\boxsize}{!}{
			\begin{tikzpicture}[trim axis left, trim axis right, every mark/.append style={mark size=2pt}]
				\begin{groupplot}[
						group style={group size=1 by 2, vertical sep=1.5cm},
						footnotesize,
						width=\figwidth,
						legend pos=south east,
						legend style={at={(0.5,1.3)}, anchor=north,legend columns=-1,
								draw=none, nodes={scale=0.9, transform shape}},
						ylabel=Accuracy (\%),
						ylabel shift = -4 pt,
						xlabel shift = -3 pt,
						ymajorgrids,
						symbolic x coords={0, 2, 4, 8, 16},
						xtick=data,
						xtick style={draw=none},
						ytick style={draw=none},
					]

					\nextgroupplot[
						xlabel=$K_x$,
						legend entries={$\epsilon_x=0.01$, $\epsilon_x=0.1$, $\epsilon_x=1$},
						#1
					]
					\foreach \xeps in {0.01, 0.1, 1.0} {
							\addplot+[thick, error bars/.cd, y dir=both, y explicit]
							table[x=x_steps, y=test/acc_mean, y error=test/acc_ci, col sep=comma] {figures/kprop/dataset=#3,x_eps=\xeps.csv};
						}

					\nextgroupplot[
						xlabel=$K_y$,
						legend entries={$\epsilon_y=0.5$, $\epsilon_y=1.0$, $\epsilon_y=2.0$},
						#2
					]

					\foreach \yeps in {0.5, 1.0, 2.0} {
							\addplot+[thick, error bars/.cd, y dir=both, y explicit]
							table[x=y_steps, y=test/acc_mean, y error=test/acc_ci, col sep=comma] {figures/kprop/dataset=#3,y_eps=\yeps.csv};

						}
				\end{groupplot}
			\end{tikzpicture}
		}
	}

	\centering
	\subfloat[Cora]{\plot[]{}{cora}}\hfil
	\subfloat[Pubmed]{\plot{}{pubmed}}\hfil
	\subfloat[Facebook]{\plot{}{facebook}}\hfil
	\subfloat[LastFM]{\plot{}{lastfm}}
	\caption{Effect of the KProp step parameter on the performance of LPGNN. The top row depicts the effect of feature KProp with $\epsilon_y=1$. The bottom row shows the effect of label KProp with $\epsilon_x=1$. The y-axis is not set to zero to focus on the trends.}
	\label{fig:kprop}
\end{figure*}

\paragraph{\textbf{Investigating the Drop algorithm.}}
In this final experiment, we investigate how using the Drop algorithm can affect the performance of LPGNN under different feature privacy budgets $\epsilon_y$ ranging with $\{0.5, 1.0, 2.0\}$, fixing $\epsilon_x=1$. We compare the result of Drop with the classic cross-entropy, where we directly train the GNN with noisy labels. We also compare with the forward correction method \cite{patrini2017making}, described in Section~\ref{sec:method}. Note that since our method does not rely on any clean validation data and is tailored for GNNs, it is not directly comparable to other general methods for deep learning with noisy labels that do not have these two characteristics. Table~\ref{tab:denoising} presents the accuracy of different learning algorithms for the three label privacy budgets. It is evident that our Drop algorithm substantially outperforms the other two methods and can remarkably increase the final accuracy compared to the baselines, especially at high-privacy regimes with severe label noise, and also on datasets like LastFM with a high number of classes.
Specifically, at $\epsilon_y=0.5$, using Drop improves the accuracy of LPGNN by over 24\%, 31\%, 6\%, and 25\% on Cora, Pubmed, Facebook, and LastFM, respectively, compared to the forward correction method. As $\epsilon_y$ increases to 2, the labels become less noisy, so the accuracy difference between Drop and the other baselines shrinks. Still, Drop can perform better or at least equally compared to the forward correction method.
This result suggests that Drop can effectively utilize the information within the graph structure to recover the actual node labels, and more importantly, it can achieve high accuracy without using any clean labels for model validation, e.g., for early stopping or hyper-parameter optimization.

\begin{table}[t]
	\centering
	\caption{Effect of Drop on the accuracy of LPGNN ($\epsilon_x=1$)}
	\label{tab:denoising}
	\small
	\sc
	\begin{tabu} to \columnwidth {l c X[c] X[c] X[c]}
		\toprule
		Dataset  & $\epsilon_y$ & Cross Entropy  & Forward Correction & Drop                     \\
		\midrule
		Cora     & 0.5          & 18.6 $\pm$ 1.3 & 18.6 $\pm$ 2.5     & \textbf{42.9 $\pm$ 1.5}  \\
		         & 1.0          & 25.5 $\pm$ 1.7 & 37.1 $\pm$ 2.5     & \textbf{69.3 $\pm$ 1.2}  \\
		         & 2.0          & 52.9 $\pm$ 2.1 & 75.1 $\pm$ 1.0     & \textbf{78.4 $\pm$ 0.7}  \\
		\midrule
		Pubmed   & 0.5          & 37.1 $\pm$ 0.9 & 38.7 $\pm$ 1.4     & \textbf{69.8 $\pm$ 0.7}  \\
		         & 1.0          & 65.4 $\pm$ 0.6 & 68.8 $\pm$ 0.7     & \textbf{74.9 $\pm$ 0.3}  \\
		         & 2.0          & 80.5 $\pm$ 0.2 & 81.0 $\pm$ 0.2     & \textbf{81.0 $\pm$ 0.2}  \\
		\midrule
		Facebook & 0.5          & 50.9 $\pm$ 4.2 & 68.9 $\pm$ 1.3     & \textbf{75.1 $\pm$ 0.6}  \\
		         & 1.0          & 55.2 $\pm$ 1.3 & 73.8 $\pm$ 1.1     & \textbf{84.9 $\pm$ 0.2}  \\
		         & 2.0          & 81.6 $\pm$ 1.2 & 88.9 $\pm$ 0.2     & \textbf{90.7 $\pm$ 0.1}  \\
		\midrule
		LastFm   & 0.5          & 21.1 $\pm$ 4.6 & 44.9 $\pm$ 5.3     & \textbf{70.0 $\pm$ 3.0 } \\
		         & 1.0          & 28.4 $\pm$ 2.5 & 58.5 $\pm$ 3.6     & \textbf{82.1 $\pm$ 1.0}  \\
		         & 2.0          & 56.8 $\pm$ 2.8 & 79.2 $\pm$ 1.3     & \textbf{85.7 $\pm$ 0.7}  \\
		\bottomrule
	\end{tabu}
\end{table}

\section{Related Work}\label{sec:review}

\paragraph{\textbf{Graph neural networks.}}
Recent years have seen a surge in applying GNNs for representation learning over graphs, and
numerous GNN models have been proposed for graph representation learning, including Graph Convolutional Networks \cite{kipf2017semi}, Graph Attention Networks \cite{velivckovic2017graph}, GraphSAGE \cite{hamilton2017inductive}, Graph Isomorphism Networks \cite{xu2018powerful}, Jumping Knowledge Networks \cite{pmlr-v80-xu18c}, Gated Graph Neural Networks \cite{li2015gated}, and so on. We refer the reader to the available surveys on GNNs~\cite{hamilton2017representation, wu2020comprehensive} for other models and discussion on their performance and applications.

\paragraph{\textbf{Local differential privacy.}}
Local differential privacy has become increasingly popular for privacy-preserving data collection and analytics, as it does not need any trusted aggregator. There have been several LDP mechanisms on estimating aggregate statistics such as frequency \cite{wang2017locally,bassily2015local,erlingsson2014rappor}, mean \cite{ding2017collecting, duchi2018minimax, wang2019collecting} heavy hitter \cite{wang2019locally}, and frequent itemset mining \cite{qin2016heavy}. There are also some works focusing on learning problems, such as probability distribution estimation \cite{acharya2018communication,duchi2018minimax,kairouz2016discrete}, heavy hitter discovery \cite{wang2019locally, bun2019heavy, bassily2017practical}, frequent new term discovery \cite{wang2018privtrie}, marginal release \cite{cormode2018marginal}, clustering \cite{nissim2018clustering}, and hypothesis testing \cite{gaboardi2018local}. Specifically, LDP frequency oracles are considered as fundamental primitives in LDP, and numerous mechanisms have been proposed \cite{erlingsson2014rappor,bassily2015local,wang2017locally,bassily2017practical,acharya2019hadamard,ye2018optimal}. Most works rely on techniques like Hadamard transform \cite{acharya2019hadamard, bassily2017practical} and hashing \cite{wang2017locally}. LDP frequency oracles are also used in other tasks, e.g., frequent itemset mining \cite{wang2018locally, qin2016heavy}, and histogram estimation \cite{wang2016mutual, wang2016using, kairouz2016discrete}.

\paragraph{\textbf{Privacy attacks on GNNs.}}
Several recent works have attempted to characterize potential privacy attacks associated with GNNs and quantify the privacy leakage of publicly released GNN models or node embeddings that have been trained on private graph data. He \etal~\cite{263820} proposed a series of link stealing attacks on a GNN model, to which the adversary has black-box access. They show that an adversary can accurately infer a link between any pair of nodes in a graph used to train the GNN model. Duddu \etal~\cite{duddu2020quantifying} presents a comprehensive study on quantifying the privacy leakage of graph embedding algorithms trained on sensitive graph data. More specifically, they introduce three major classes of privacy attacks on GNNs, namely membership inference, graph reconstruction, and attribute inference attack, under practical threat models and adversary assumptions. Finally, Wu \etal~\cite{wu2020model} propose a model extraction attack against GNNs by generating legitimate-looking queries as the normal nodes among the target graph, and then utilizing the query responses accessible structure knowledge to reconstruct the model. Overall, these works underline many privacy risks associated with GNNs and demonstrate the vulnerability of these models to various privacy attacks.

\paragraph{\textbf{Privacy-preserving GNN models.}}
While there is a growing interest in both theory and applications of GNNs, there have been relatively few attempts to provide privacy-preserving graph representation learning algorithms. Xu \etal~\cite{xu2018dpne} proposed a differentially private graph embedding method by applying the objective perturbation on the loss function of matrix factorization. Zhang and Ni~\cite{zhang2019graph} proposed a differentially private perturbed gradient descent method based on Lipschitz condition~\cite{jha2013testing} for matrix factorization-based graph embedding. Both of these methods target classic graph embedding algorithms and not GNNs.
Li \etal~\cite{li2020adversarial} presented a graph adversarial training framework that integrates disentangling and purging mechanisms to remove users' private information from learned node representations. Liao \etal~\cite{liao2020graph} also follow an adversarial learning approach to address the attribute inference attack on GNNs, where they introduce a minimax game between the desired graph feature encoder and the worst-case attacker.
However, both of these works assume that the server has complete access to the private data, which is as opposed to our problem setting.

There are also recent approaches that attempted to address privacy in GNNs using federated and split learning. Mei \etal~\cite{9005983} proposed a GNN based on structural similarity and federated learning to hide content and structure information. Zhou \etal~\cite{zhou2020privacy} tackled the problem of privacy-preserving node classification by splitting the computation graph of a GNN between multiple data holders and use a trusted server to combine the information from different parties and complete the training. However, as opposed to our method, these approaches rely on a trusted third party for model aggregation, and their privacy protection is not formally guaranteed.
Finally, Jiang \etal~\cite{jiang2020federated} proposed a distributed and secure framework to learn
the object representations in video data from graph sequences based on GNN and federated learning, and design secure aggregation primitives to protect privacy in federated learning. However, they assume that each party owns a series of graphs (extracted from video data), and the server uses federated learning to learn an inductive GNN over this distributed dataset of graphs, which is a different problem setting than the node data privacy we studied.

\section{Conclusion}\label{sec:conclusion}

In this paper, we presented a locally private GNN to address node data privacy, where graph nodes have sensitive data that are kept private, but a central server could leverage them to train a GNN for learning rich node representations. To this end, we first proposed the \emph{multi-bit mechanism}, a multidimensional \ldp algorithm that allows the server to privately collect node features and estimate the first-layer graph convolution of the GNN using the noisy features. Then, to further decrease the estimation error, we introduced KProp, a simple graph convolution layer that aggregates features from higher-order neighbors, which is prepended to the backbone GNN. Finally, to learn the model with perturbed labels, we proposed a learning algorithm called Drop that utilizes KProp for label denoising. Experimental results over real-world graph datasets on node classification demonstrated that the proposed framework could maintain an appropriate privacy-utility trade-off.

The concept of privacy-preserving graph representation learning is a novel field with many potential future directions that can go beyond node data privacy, such as link privacy and graph-level privacy. For the presented work,
several future trends and improvements are imaginable. Firstly, in this paper, we protected the privacy of node features and labels, but the graph topology is left unprotected. Therefore, an important future work is to extend the current setting to preserving the graph structure as well.
Secondly, we would like to explore other neighborhood expansion mechanisms that are more effective than the proposed KProp. Another future direction is to develop more rigorous algorithms for learning with differentially private labels, which is left unexplored for the case of GNNs.
Finally, an interesting future work would be to combine the proposed LPGNN with deep graph learning algorithms to address privacy-preserving classification over non-relational datasets with low communication cost.

\begin{acks}
	This work was supported by the \grantsponsor{snsf}{Swiss National Science Foundation (SNSF)}{http://www.snf.ch/en/Pages/default.aspx} through the Dusk2Dawn project (Sinergia program) under grant number \grantnum{snsf}{173696}. 
	Additional support was provided by the European Commission under European Horizon 2020 Programme, grant number 951911, AI4Media project.
	We would like to thank Emiliano De Cristofaro, Hamed Haddadi, Nikolaos Karalias, and Mohammad Malekzadeh for their helpful comments on eariler drafts of this paper.
\end{acks}

\bibliographystyle{ACM-Reference-Format}
\bibliography{paper}

\appendix



\section{Deferred Theoretical Arguments}\label{sec:proof}

\subsection{Theorem~\ref{th:dp}}
\begin{proof}
	Let $\mathcal{M}(\mathbf{x})$ denote the multi-bit encoder (Algorithm~\ref{alg:mbm}) applied on the input vector $\mathbf{x}$. Let $\mathbf{x}^* = \mathcal{M}(\mathbf{x})$ be the encoded vector corresponding to $\mathbf{x}$. We need to show that for any two input features $\mathbf{x}_1$ and $\mathbf{x}_2$, we have  $\frac{\Pr\left[\mathcal{M}(\mathbf{x}_1) = \mathbf{x}^*\right]}{\Pr\left[\mathcal{M}(\mathbf{x}_2) = \mathbf{x}^*\right]} \le e^\epsilon$.

	According to Algorithm~\ref{alg:mbm}, for any dimension $i\in\{1,2,\dots,d\}$, it can be easily seen that $x^*_i\in\{-1,0,1\}$. The case $x^*_i = 0$ occurs when $i\notin\mathcal{S}$ with probability $1-\frac{m}{d}$, therefore:
	\begin{equation}\label{eq:zero}
		\frac{\Pr\left[\mathcal{M}(\mathbf{x}_1)_i = 0\right]}{\Pr\left[\mathcal{M}(\mathbf{x}_2)_i = 0\right]} = \frac{1-m/d}{1-m/d} = 1 \le e^\epsilon, \quad \forall \epsilon > 0
	\end{equation}
	According to Algorithm~\ref{alg:mbm}, in the case of $x^*_i\in\{-1,1\}$, we see that the probability of getting $x^*_i = 1$ ranges from $\frac{m}{d}\cdot\frac{1}{e^\ek + 1}$ to $\frac{m}{d}\cdot\frac{e^\ek}{e^\ek + 1}$ depending on the value of $x_i$. Analogously, the probability of  $x^*_i = -1$ also varies from $\frac{m}{d}\cdot\frac{1}{e^\ek + 1}$ to $\frac{m}{d}\cdot\frac{e^\ek}{e^\ek + 1}$. Therefore:
	\begin{align*}
		\frac{\Pr\left[\mathcal{M}(\mathbf{x}_1)_i \in\{-1,1\}\right]}{\Pr\left[\mathcal{M}(\mathbf{x}_2)_i \in\{-1,1\}\right]} & \le \frac{\max\Pr\left[\mathcal{M}(\mathbf{x}_1)_i \in\{-1,1\}\right]}{\min\Pr\left[\mathcal{M}(\mathbf{x}_2)_i \in\{-1,1\}\right]}   \\
		                                                                                                                        & \le \frac{\frac{m}{d}\cdot\frac{e^\ek}{e^\ek + 1}}{\frac{m}{d}\cdot\frac{1}{e^\ek + 1}} \le e^\ek \numberthis \label{eq:plusminusone}
	\end{align*}
	Consequently, we have:
	\begin{align*}
		\frac{\Pr\left[\mathcal{M}(\mathbf{x}_1) = \mathbf{x}^*\right]}{\Pr\left[\mathcal{M}(\mathbf{x}_2) = \mathbf{x}^*\right]} & = \prod_{i=1}^{d}\frac{\Pr\left[\mathcal{M}(x_1)_i = {x}^*_i\right]}{\Pr\left[\mathcal{M}({x}_2)_i = {x}^*_i\right]}                                           \\
		                                                                                                                          & = \prod_{j|x^*_j=0}\frac{\Pr\left[\mathcal{M}(x_1)_j = 0\right]}{\Pr\left[\mathcal{M}({x}_2)_j = 0\right]}                                                     \\ &\quad \times \prod_{k|x^*_k\in\{-1,1\}}\frac{\Pr\left[\mathcal{M}(x_1)_k\in\{-1,1\}\right]}{\Pr\left[\mathcal{M}({x}_2)_k\in\{-1,1\}\right]} \numberthis \\
		                                                                                                                          & = \prod_{x^*_k\in\{-1,1\}}\frac{\Pr\left[\mathcal{M}(x_1)_k\in\{-1,1\}\right]}{\Pr\left[\mathcal{M}({x}_2)_k\in\{-1,1\}\right]} \numberthis \label{eq:elimone} \\
		                                                                                                                          & \le \prod_{x^*_k\in\{-1,1\}}e^\ek \numberthis \label{eq:changeem}                                                                                              \\
		                                                                                                                          & \le e^\epsilon \numberthis \label{eq:finaleps}
	\end{align*}
	which concludes the proof. In the above, (\ref{eq:elimone}) and (\ref{eq:changeem}) follows from applying (\ref{eq:zero}) and (\ref{eq:plusminusone}), respectively, and (\ref{eq:finaleps}) follows from the fact that exactly $m$ number of input features result in non-zero output.

\end{proof}

\subsection{Proposition~\ref{prop:unbiased}}
We first establish the following lemma and then prove Proposition~\ref{prop:unbiased}:

\begin{lemma}\label{lem:expvar}
	Let $\mathbf{x}^*$ be the output of Algorithm~\ref{alg:mbm} on the input vector $\mathbf{x}$. For any dimension $i\in\{1,2,\dots,d\}$, we have:
	\begin{equation}
		\E\left[x^*_i\right] = \frac{m}{d}\cdot\frac{e^\ek-1}{e^\ek+1}\cdot\left( 2\cdot\frac{x_i - \alpha}{\beta-\alpha} -1 \right)
	\end{equation}
	and
	\begin{equation}
		Var\left[x^*_i\right] = \frac{m}{d} - \left[\frac{m}{d}\cdot\frac{e^\ek-1}{e^\ek+1}\cdot\left( 2\cdot\frac{x_i - \alpha}{\beta-\alpha} -1 \right)\right]^2
	\end{equation}
\end{lemma}
\begin{proof}
	For the expectation, we have:
	\begin{align*}
		\E\left[x^*_i\right] & = \E\left[x^*_i\mid s_i=0\right]\Pr(s_i=0) + \E\left[x^*_i\mid s_i=1\right]\Pr(s_i=1)     \\
		                     & = \frac{m}{d}\cdot\left(2\E\left[t_i\right] - 1\right) \numberthis \label{eq:lm1:proof:1}
	\end{align*}
	Since $t_i$ is a Bernoulli random variable, we have:
	\begin{equation} \label{eq:lm1:proof:2}
		\E\left[t_i\right] = \frac{1}{e^\ek + 1} + \frac{x_{i} - \alpha}{\beta - \alpha}\cdot\frac{e^\ek - 1}{e^\ek + 1}
	\end{equation}
	Combining (\ref{eq:lm1:proof:1}) and (\ref{eq:lm1:proof:2}) yields:
	\begin{align*}
		\E\left[x^*_i\right] & = \frac{m}{d}\cdot\left[2\left(\frac{1}{e^\ek + 1} + \frac{x_{i} - \alpha}{\beta - \alpha}\cdot\frac{e^\ek - 1}{e^\ek + 1}\right) - 1\right] \\
		                     & = \frac{m}{d}\cdot\left[\frac{1 - e^\ek}{e^\ek + 1} + 2\cdot\frac{x_{i} - \alpha}{\beta - \alpha}\cdot\frac{e^\ek - 1}{e^\ek + 1}\right]     \\
		                     & = \frac{m}{d}\cdot\frac{e^\ek-1}{e^\ek+1}\cdot\left( 2\cdot\frac{x_i - \alpha}{\beta-\alpha} -1 \right) \numberthis \label{eq:lm1:proof:exp}
	\end{align*}
	For the variance, we have:
	\begin{align*}
		Var\left[x^*_i\right] & = \E\left[\left(x^*_i\right)^2\right] - \E\left[x^*_i\right]^2                           \\
		                      & = \E\left[\left(x^*_i\right)^2\mid s_i=0\right]\Pr(s_i=0)                                \\
		                      & \quad + \E\left[\left(x^*_i\right)^2\mid s_i=1\right]\Pr(s_i=1) - \E\left[x^*_i\right]^2
	\end{align*}
	Given $s_i = 1$, we have $x^*_i=\pm1$, and thus $\left(x^*_i\right)^2=1$. Therefore, combining with (\ref{eq:lm1:proof:exp}), we get:
	\begin{equation}
		Var\left[x^*_i\right] = \frac{m}{d} - \left[\frac{m}{d}\cdot\frac{e^\ek-1}{e^\ek+1}\cdot\left( 2\cdot\frac{x_i - \alpha}{\beta-\alpha} -1 \right)\right]^2
	\end{equation}
\end{proof}

Now we prove Proposition~\ref{prop:unbiased}.
\begin{proof}
	We need to show that $\E\left[x^\prime_{v,i}\right] = x_{v,i}$ for any $v\in\mathcal{V}$ and any dimension $i\in\{1,2,\dots,d\}$.
	\begin{equation}
		\E\left[x^\prime_{v,i}\right] = \frac{d(\beta-\alpha)}{2m}\cdot\frac{e^\ek + 1}{e^\ek - 1}\cdot\E\left[{x}^*_{v,i}\right] + \frac{\alpha + \beta}{2}
	\end{equation}
	Applying Lemma~\ref{lem:expvar} yields:
	\begin{align*}
		\E\left[x^\prime_{v,i}\right] & = \frac{d(\beta-\alpha)}{2m}\frac{e^\ek + 1}{e^\ek - 1}\left[\frac{m}{d}\frac{e^\ek-1}{e^\ek+1}\left( 2\frac{x_{v,i} - \alpha}{\beta-\alpha} -1 \right)\right] \\
		                              & \quad + \frac{\alpha + \beta}{2}                                                                                                                               \\
		                              & = \frac{\beta-\alpha}{2}\left( 2\frac{x_{v,i} - \alpha}{\beta-\alpha} -1 \right) + \frac{\alpha + \beta}{2}                                                    \\
		                              & = x_{v,i} - \alpha - \frac{\beta - \alpha}{2} + \frac{\alpha + \beta}{2} = x_{v,i}
	\end{align*}
\end{proof}

\subsection{Proposition~\ref{prop:var}}
\begin{proof}
	According to (\ref{eq:est}), the variance of $x^\prime_{v,i}$ can be written in terms of the variance of $x^*_{v,i}$ as:
	\[ Var\left[x^\prime_{v,i}\right] = \left[\frac{d(\beta-\alpha)}{2m}\cdot\frac{e^\ek + 1}{e^\ek - 1}\right]^2\cdot Var\left[x^*_{v,i}\right]   \]
	Applying Lemma~\ref{lem:expvar} yields:
	\begin{align*}
		Var\left[x^\prime_{v,i}\right] & = \left[\frac{d(\beta-\alpha)}{2m}\cdot\frac{e^\ek + 1}{e^\ek - 1}\right]^2                                                                                    \\
		                               & \quad\times\left(\frac{m}{d} - \left[\frac{m}{d}\cdot\frac{e^\ek-1}{e^\ek+1}\cdot\left( 2\cdot\frac{x_{v,i} - \alpha}{\beta-\alpha} -1 \right)\right]^2\right) \\
		                               & = \frac{d}{m}\cdot\left(\frac{\beta-\alpha}{2}\cdot\frac{e^\ek + 1}{e^\ek - 1}\right)^2                                                                        \\
		                               & \quad - \left[ \frac{\beta-\alpha}{2}\cdot\left( 2\cdot\frac{x_{v,i} - \alpha}{\beta-\alpha} -1 \right) \right]^2                                              \\
		                               & = \frac{d}{m}\cdot\left(\frac{\beta-\alpha}{2}\cdot\frac{e^\ek + 1}{e^\ek - 1}\right)^2 - \left({x}_{v,i} - \frac{\alpha + \beta}{2}\right)^2
	\end{align*}
\end{proof}

\subsection{Proposition~\ref{prop:optm}}

\begin{proof}
	We look for a value of $m$ that minimizes the variance of the multi-bit rectifier defined by (\ref{eq:est}), i.e., $Var[x^\prime_{v,i}]$, for any arbitrary node $v\in\mathcal{V}$ and any arbitrary dimension $i\in\{1,,2,\dots,d\}$. However, based on Proposition~\ref{prop:var}, $Var[x^\prime_{v,i}]$ depends on the private feature $x_{v,i}$, which is unknown to the server. Therefore, we find the optimal $m$, denoted by $m^\star$, by minimizing the upperbound of the variance:
	\begin{equation}
		m^\star = \arg\min_m \max_x Var[x^\prime]
	\end{equation}
	where we omitted the node $v$ and dimension $i$ subscripts for simplicity. From Proposition~\ref{prop:var}, it can be easily seen that the variance is maximized when $x=\frac{\alpha+\beta}{2}$, which yields:
	\begin{align*}
		\max_x Var[x^\prime] & = \frac{d}{m}\cdot\left(\frac{\beta-\alpha}{2}\cdot\frac{e^\ek + 1}{e^\ek - 1}\right)^2 \numberthis \label{eq:maxvar}     \\
		                     & = C\cdot z\cdot\left(\frac{e^z + 1}{e^z - 1}\right)^2 = C\cdot z\cdot \coth^2(\frac{z}{2}) \numberthis \label{eq:maxvarz}
	\end{align*}
	where we set $z = \frac{\epsilon}{m}$ and $C=\frac{d}{\epsilon}\cdot\left(\frac{\beta-\alpha}{2}\right)^2$, and $\coth(.)$ is the hyperbolic cotangent. Therefore, minimizing (\ref{eq:maxvar}) with respect to $m$ is equivalent to minimizing (\ref{eq:maxvarz}) with respect to $z$, and then recover $m^\star$ as $\frac{\epsilon}{z^\star}$, where $z^\star$ is the optimal $z$ minimizing (\ref{eq:maxvarz}). More formally:
	\begin{align*}
		z^\star & = \arg\min_z \left[C\cdot z\cdot \coth^2(\frac{z}{2})\right] \\
		        & = \arg\min_z\left[ z\cdot \coth^2(\frac{z}{2})\right]        \\
	\end{align*}
	where the constant $C$ were dropped as it does not depend on $z$. The function $f(z) =  z\cdot \coth^2(\frac{z}{2})$ is a convex function with a single minimum on $(0,\infty)$, as shown in Figure~\ref{fig:z}. Taking the derivative of $f(.)$ with respect to $z$ and set it to zero gives us the minimum:
	\begin{align*}
		f'(z) & = \frac{d}{dz} z\cdot\coth^2(\frac{z}{2}) = \coth(\frac{z}{2})\left[\coth(\frac{z}{2}) - z\cdot \csch^2(\frac{z}{2})\right] = 0
	\end{align*}
	and then we have:
	\begin{equation}
		z = \frac{\coth(\frac{z}{2})}{\csch^2(\frac{z}{2})} = \frac{\sinh(z)}{2}
	\end{equation}
	Solving the above equation yields $z^\star \simeq 2.18$, and therefore we have $m^\star = \frac{\epsilon}{2.18}$. However, $m$ should be an integer value between 1 and $d$. To enforce this, we set:
	\begin{equation}
		m^\star = \max(1, \min(d, \, \left\lfloor \frac{\epsilon}{2.18}\right\rfloor))
	\end{equation}
\end{proof}

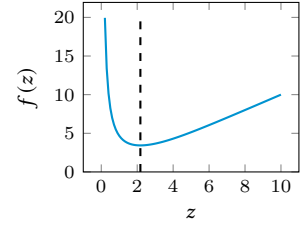
\begin{figure}[t]
	\centering
	\begin{tikzpicture}
		\begin{axis}[
				footnotesize,
				width=\figwidth,
				xlabel=$z$,
				ylabel={$f(z)$},
				domain=0:10,
				restrict y to domain=0:20,
				ymin=0,
				samples=100,
			]
			\addplot[mark=none, color=ACMBlue, thick] {x*((e^x+1)/(e^x-1))^2};
			\draw [dashed, thick] (2.1773,0) -- (2.1773,20);
		\end{axis}
	\end{tikzpicture}
	\caption{Plotting $f(z)=z\cdot \coth^2(\frac{z}{2})$. The gray dashed line indicate the location of the minimum.}
	\label{fig:z}
\end{figure}

\subsection{Corollary~\ref{cor:unbiased}}
\begin{proof}
	We need to show that the following holds for any node $v\in\mathcal{V}$:
	\[\E\left[\widehat{\mathbf{h}}_{\N(v)}\right] = {\mathbf{h}}_{\N(v)}\]
	The left hand side of the above can be written as:
	\[ \E\left[\widehat{\mathbf{h}}_{\N(v)}\right] = \E\left[\agg{}\left(\{\mathbf{x}^\prime_u, \forall u \in \mathcal{N}(v)\}\right)\right] \]
	Since \agg{} is linear, due to the linearity of expectation, the expectation sign can be moved inside \agg{}:
	\[\E\left[\widehat{\mathbf{h}}_{\N(v)}\right] = \agg{}\left(\{\E\left[\mathbf{x}^\prime_u\right], \forall u \in \mathcal{N}(v)\}\right)\]
	Finally, by Proposition~\ref{prop:unbiased}, we have:
	\[\E\left[\widehat{\mathbf{h}}_{\N(v)}\right] = \agg{}\left(\{\mathbf{x}_u, \forall u \in \mathcal{N}(v)\}\right) = {\mathbf{h}}_{\N(v)}\]
\end{proof}

\subsection{Proposition~\ref{prop:error}}
\begin{proof}
	According to (\ref{eq:est}) and depending on Algorithm~\ref{alg:mbm}'s output, for any node $u\in\mathcal{V}$ and any dimension $i\in\{1,2,\dots,d\}$, we have:
	\begin{equation*}
		x^\prime_{u,i} =
		\begin{cases}
			\frac{\alpha + \beta}{2} - \frac{d(\beta-\alpha)}{2m}\cdot\frac{e^\ek + 1}{e^\ek - 1} \quad & \text{ if } x^*_{u,i} = -1 \\
			\frac{\alpha + \beta}{2} \quad                                                              & \text{ if } x^*_{u,i} = 0  \\
			\frac{\alpha + \beta}{2} + \frac{d(\beta-\alpha)}{2m}\cdot\frac{e^\ek + 1}{e^\ek - 1} \quad & \text{ if } x^*_{u,i} = 1
		\end{cases}
	\end{equation*}
	and therefore
	\[x^\prime_{u,i}\in[\frac{\alpha + \beta}{2} - C, \frac{\alpha + \beta}{2} + C]\]
	where
	\begin{equation}
		C =\frac{d(\beta-\alpha)}{2m}\cdot\frac{e^\ek + 1}{e^\ek - 1}
	\end{equation}
	Therefore, considering that $x_{u,i}\in[\alpha, \beta]$, we get:
	\begin{equation}\label{eq:prop:error:proof:1}
		\left|x^\prime_{u,i} - x_{u,i}\right|\le\frac{\beta-\alpha}{2} + C
	\end{equation}
	and also by Proposition~\ref{prop:unbiased}, we know that
	\begin{equation}
		\E\left[x^\prime_{u,i} - x_{u,i}\right] = 0
	\end{equation}
	On the other hand, using the mean aggregator function, for any node $v\in\mathcal{V}$ and any dimension $i\in\{1,2,\dots,d\}$, we have:
	\begin{align*}
		({\mathbf{h}}_{\N(v)})_i = \frac{1}{|\N(v)|}\sum_{u\in\N(v)}x_{u,i} \\
		(\widehat{\mathbf{h}}_{\N(v)})_i = \frac{1}{|\N(v)|}\sum_{u\in\N(v)}x^\prime_{u,i} \numberthis \label{eq:prop:error:proof:2}
	\end{align*}
	Considering~\ref{eq:prop:error:proof:1} to~\ref{eq:prop:error:proof:2} and using the Bernstein inequality, we have:
	\begin{align*}
		 & \Pr\left[\left|(\widehat{\mathbf{h}}_{\N(v)})_i - ({\mathbf{h}}_{\N(v)})_i\right|\ge\lambda\right]                                                                           \\
		 & \quad= \Pr\left[\left|\sum_{u\in\N(v)}(x^\prime_{u,i} - x_{u,i})\right|\ge\lambda|N(v)|\right]                                                                               \\
		 & \quad\le 2\cdot\exp\left\{-\frac{\lambda^2|\N(v)|}{\frac{2}{|\N(v)|}\sum_{u\in\N(v)}Var[x^\prime_{u,i} - x_{u,i}] + \frac{2}{3}\lambda(\frac{\beta-\alpha}{2} + C)} \right\} \\
		 & \quad= 2\cdot\exp\left\{-\frac{\lambda^2|\N(v)|}{2Var[x^\prime_{u,i}] + \frac{2}{3}\lambda(\frac{\beta-\alpha}{2} + C)} \right\} \numberthis \label{eq:prop:error:proof:5}   \\
	\end{align*}
	We can rewrite the variance of $x^\prime_{u,i}$ in terms of $C$ as:
	\begin{equation}
		Var[x^\prime_{u,i}] = \frac{m}{d}C^2 - \left({x}_{v,i} - \frac{\alpha + \beta}{2}\right)^2
	\end{equation}
	The asymptotic expressions involving $\epsilon$ are evaluated in ${\epsilon\rightarrow0}$, which yields:
	\begin{equation}\label{eq:prop:error:proof:3}
		C = \frac{d(\beta-\alpha)}{2m}\mathcal{O}(\frac{m}{\epsilon}) = \mathcal{O}(\frac{d}{\epsilon})
	\end{equation}
	and therefore we have:
	\begin{align*}
		Var[x^\prime_{u,i}] & = \frac{m}{d}\left(\mathcal{O}(\frac{d}{\epsilon})\right)^2 - \left({x}_{v,i} - \frac{\alpha + \beta}{2}\right)^2 = \mathcal{O}(\frac{md}{\epsilon^2}) \numberthis \label{eq:prop:error:proof:4}
	\end{align*}
	Substituting (\ref{eq:prop:error:proof:3}) and (\ref{eq:prop:error:proof:4}) in (\ref{eq:prop:error:proof:5}), we have:
	\begin{align*}
		\Pr\left[\left|(\widehat{\mathbf{h}}_{\N(v)})_i - ({\mathbf{h}}_{\N(v)})_i\right|\ge\lambda\right] \le 2\cdot\exp\left\{-\frac{\lambda^2|\N(v)|}{\mathcal{O}(\frac{md}{\epsilon^2}) + \lambda\mathcal{O}(\frac{d}{\epsilon})} \right\}
	\end{align*}
	According to the union bound, we have:
	\begin{align*}
		 & \Pr\left[\max_{i\in\{1,\dots,d\}}\left|(\widehat{\mathbf{h}}_{\N(v)})_i - ({\mathbf{h}}_{\N(v)})_i\right|\ge\lambda\right]              \\
		 & \quad= \bigcup_{i=1}^d \Pr\left[\left|(\widehat{\mathbf{h}}_{\N(v)})_i - ({\mathbf{h}}_{\N(v)})_i\right|\ge\lambda\right]               \\
		 & \quad\le \sum_{i=1}^{d} \Pr\left[\left|(\widehat{\mathbf{h}}_{\N(v)})_i - ({\mathbf{h}}_{\N(v)})_i\right|\ge\lambda\right]              \\
		 & \quad= 2d\cdot\exp\left\{-\frac{\lambda^2|\N(v)|}{\mathcal{O}(\frac{md}{\epsilon^2}) + \lambda\mathcal{O}(\frac{d}{\epsilon})} \right\}
	\end{align*}
	To ensure that $\max_{i\in\{1,\dots,d\}}\left|(\widehat{\mathbf{h}}_{\N(v)})_i - ({\mathbf{h}}_{\N(v)})_i\right|<\lambda$ holds with at least $1-\delta$ probability, it is sufficient to set
	\begin{equation}
		\delta = 2d\cdot\exp\left\{-\frac{\lambda^2|\N(v)|}{\mathcal{O}(\frac{md}{\epsilon^2}) + \lambda\mathcal{O}(\frac{d}{\epsilon})} \right\}
	\end{equation}
	Solving the above for $\lambda$, we get:
	\begin{equation}
		\lambda = \mathcal{O}\left(\frac{\sqrt{d \log(d/\delta)}}{\epsilon \sqrt{|\N(v)|}}\right)
	\end{equation}
\end{proof}

\subsection{Corollary~\ref{cor:dp}}
\begin{proof}
	The training steps in Algorithm~\ref{th:dp} only process the output of the multi-bit encoder and the randomized response mechanism, which respectively provide $\epsilon_x\text{-LDP}$ and $\epsilon_y\text{-LDP}$ for each node. Private node features and labels are not used anywhere else in the algorithm except by the multi-bit encoder and the randomized response mechanism. Since Algorithm~\ref{alg:dpgnn} calls the encoder and randomized response only once per node, and due to the basic composition theorem and the robustness of differentially private algorithms to post-processing~\cite{dwork2014algorithmic}, Algorithm~\ref{alg:dpgnn} satisfies $(\epsilon_x+\epsilon_y)\text{-LDP}$ for each node.
\end{proof}

\end{document}
\endinput